\documentclass{IEEEtran}

\usepackage{graphicx} 
\usepackage{amsfonts}
\usepackage{amssymb}
\usepackage{amsthm}
\usepackage{mathtools}
\usepackage{bbm}

\usepackage{amsmath,amsthm,amssymb,mathrsfs} 

 \usepackage{amsmath,amsthm,bbm,amssymb}

\DeclareMathOperator{\logg}{\log\log}

\def\R{\mathbb{R}}

\def\cH{\mathcal{H}}

\def\cM{\mathcal{M}}

\def\cX{\mathcal{X}}

\newcommand{\E}{\mathbb{E}} 

\newcommand{\given}{\;|\;}
\newcommand{\mean}[1] {\E\left\{{#1}\right\}}

\newcommand{\ind}{\boldsymbol{\mathbbm{1}}} 
\newcommand{\indf}[1]{\ind\set{#1}} 

\newcommand{\set}[1]{\left\{#1\right\}}

\newcommand{\param}[1]{\left(#1\right)}
\newcommand{\abs}[1] {\left| {#1}\right|}

\newcommand{\prob}[1]{\mathbb{P}\left(#1\right)}

\newcommand{\cprob}[2]{\mathbb{P}\left(#1\given #2\right)} 

\newcommand{\eps}{\epsilon}

\providecommand{\setthms}[1]{#1}
\setthms{
\newtheorem{lem}{Lemma}[section]
\newtheorem{thm}[lem]{Theorem}

\theoremstyle{definition}

}

\newcommand{\ninf}{n\to\infty}

\newcommand{\fmax}{f_{\max}}
\newcommand{\fmin}{f_{\min}}

\newcommand{\limninf}{\lim_{\ninf}}

\newcommand{\bs}{\backslash}

\numberwithin{equation}{section}

\def\bsplit#1\esplit{\begin{split} #1 \end{split} }
\def\splitb#1\splite{\begin{split} #1 \end{split} }
\def\beq#1\eeq{\begin{equation} #1 \end{equation}}
\def\eqb#1\eqe{\begin{equation} #1 \end{equation}}

\usepackage[pagebackref,breaklinks,colorlinks]{hyperref}
\usepackage{xfrac,multirow}
\usepackage{authblk}
\usepackage{subcaption}

\title{A Universal Nearest-Neighbor Estimator for Intrinsic Dimensionality }
\author[1]{Eng-Jon Ong}
\author[1]{Omer Bobrowski}
\author[2]{Gesine Reinert}
\author[1]{Primoz Skraba}

\affil[1]{School of Mathematical Sciences, Queen Mary University of London}
\affil[2]{Department of Statistics, Oxford University }

\theoremstyle{definition}
\newtheorem{theorem}{Theorem}

\newtheorem{remark}[theorem]{Remark}

\newtheoremstyle{case}{}{}{}{}{}{:}{ }{}
\theoremstyle{case}

\newcommand{\gr}[1]{{\color{magenta}#1}}

\def\name{L2N2}
\begin{document}

\maketitle

\begin{abstract}
Estimating the intrinsic dimensionality (ID) of data is a fundamental problem in machine learning and computer vision, providing insight into the true degrees of freedom underlying high-dimensional observations. Existing methods often rely on 
geometric or distributional assumptions and can significantly fail when these assumptions are violated. In this paper, we introduce a novel ID estimator based on nearest-neighbor distance ratios that involves simple calculations and achieves state-of-the-art results. Most importantly, we provide a theoretical analysis proving that our estimator is \emph{universal}, namely, it converges to the true ID independently of the distribution generating the data. We present experimental results on benchmark manifolds and real-world datasets to demonstrate the performance of our estimator.
\end{abstract}

\section{Introduction}

The \emph{manifold hypothesis} \cite{whiteley2025statistical} states that high-dimensional data is most often concentrated on or around a lower dimensional space.
The dimension of this lower dimensional structure is referred to as the \emph{intrinsic dimensionality} (ID) of the data,  measuring the degrees of freedom, or the number of latent variables capturing the  structure of the data.
Estimating the intrinsic dimensionality of data is a pivotal task in data analysis and machine learning \cite{magnitude,PH,Costa}, computer vision \cite{gong2019intrinsic, Pope2021TheID, gamper2021multiple, lorenz2023detecting}, materials science \cite{lung1988fractal,zhou2021local}, signal processing \cite{carter2009local,little2009estimation}, dynamical systems \cite{grassberger1983characterization,takens1981numerical}, and other areas. Existing methods face significant challenges such as the curse of dimensionality, scale dependence,  and sensitivity to the underlying distribution.

In this paper we present {\name} -- a new ID estimator based on the loglog (L2) values of the nearest-neighbor (N2) distance ratio. The key properties of {\name} are: (a) it is computationally efficient, (b) it is supported by statistical theory proving that it converges to the true ID, regardless of what the data distribution is, and (c) it achieves, and even surpasses, state-of-the-art results.

\subsection{Related Work}

Dimensionality estimation is an extensively studied problem; we refer the reader to the  comprehensive surveys~\cite{Campadelli, binnie2025surveydimensionestimationmethods} for more details. There are several approaches to ID estimation, each with their own strengths and limitations, and we will briefly mention a few here.

A classical approach is to estimate the asymptotic slope of a given statistic as a function of the sample size. Examples include the correlation dimension \cite{grassberger1983characterization},
the box-counting dimension \cite{bouligand1929notion},
and the minimum-spanning-tree dimension \cite{costa2004geodesic}.
A second approach is based on various projection methods, looking for the number of directions with largest variance.
This includes PCA and local PCA \cite{Fan2010IntrinsicDE}, MDS \cite{MDS}, Isomap \cite{ISOMAP}, diffusion maps \cite{coifman2006diffusion}, and more.

The approach most relevant to our work is based on the relationship between the nearest-neighbor distances. Levina and Bickel~\cite{levina} proposed a maximum likelihood ID estimator using the nearest-neighbor distances. This method assumes that locally the points follow the distribution of a homogeneous Poisson point process. 
The TwoNN (Two Nearest-Neighbors) method, by Facco et al.~\cite{Facco}, builds on similar principles, using the ratio of the first and second nearest-neighbor distances to obtain a computationally efficient estimator. 
The GriDE (Generalized Ratios ID Estimator) method,  
by Denti et al.~\cite{denti}, generalizes the TwoNN estimator using  multiple nearest-neighbor distances. This estimator demonstrates improved stability and accuracy, particularly in noisy and heterogeneous data distributions.
The DANCo (Dimensionality from Angle and Norm Concentration) method, by Ceruti et al.~\cite{Ceruti}, includes additional information from the angles between nearest neighbors. 

Recent work further refines nearest-neighbor and likelihood-based approaches. Qiu et al.~\cite{Qiu_under} study the finite-sample underestimation bias of nearest-neighbor estimators and propose an explicit correction that improves empirical accuracy. 
From a representation-learning perspective, Kärkkäinen~\cite{Karkkainen} introduces an additive autoencoder framework in which the intrinsic dimension is inferred from reconstruction behavior as the latent dimension varies. 
Likelihood-based methods have also been explored: Horvat et al.~\cite{Horvat} estimate intrinsic dimensionality using normalizing flows, leveraging invertible generative models to analyze the effective dimensional structure of the data distribution, while Tempczyk~\cite{Tempczyk} proposes LIDL, a local likelihood-based estimator that connects intrinsic dimension to small-scale density behavior.
\subsection{Context and Contributions}
 A key  property for any dimensionality estimator is \emph{scale invariance}, namely, rescaling the data should have no effect on the results. Aside from making it independent to the choice of measurement units, which is clearly desirable, we focus here on a much less obvious, yet very powerful, consequence. Under some additional  conditions, scale invariance may lead to the property of \emph{universality} \cite{bobrowski2023universal,bobrowski2024universalityrandompersistenthomology} -- the limiting distribution of the estimator is independent of the distribution of the data. Our {\name} estimator was designed to take advantage of this phenomenon, providing excellent performance under very minimal assumptions on the data.
 
 %
%
The contributions of the paper are:
\begin{itemize}
    \item We derive {\name} -- a new method for estimating  ID from ratios of nearest neighbor distances. 
    \item We provide a rigorous theoretical analysis
    of {\name}, and give the precise relationship between the ID and the computed statistic. 
    Most importantly, we show that the proven guarantees for {\name} do not require knowing the distribution of the
    data, 
    and that they apply to a vast class of input distributions with only mild assumptions.
    \item As the theoretical results are asymptotic, we provide a method to deal with finite sample effects.
     We show that {\name} achieves, and even exceeds, the state-of-the-art performance on benchmarking datasets. We also apply {\name} to real-world datasets, showing consistent results 
    with commonly accepted dimensionalities. 
\end{itemize}
Our proofs apply for data supported on $C^1$ manifolds with bounded densities, which already encompass a broad class of practical models. We believe the mechanism driving universality is substantially more general and should extend to settings such as fractal measures and stratified spaces. However, showing this rigorously remains a direction for future work.


\vspace{5pt}
The rest of the paper is organized as follows. Section \ref{sec:dim_est_method} introduces our new dimensionality-estimation method, {\name}. Section \ref{sec:theory} presents the theoretical analysis supporting {\name}. Section \ref{sec:exp} outlines the experimental setup, and Section \ref{sec:res} discusses the results. Section \ref{sec:exp_discussion} offers a concluding discussion.
Additional derivations and empirical results for the constants in our method are given in the Appendices.

\section{NN-Ratio Based Dimensionality Estimation}\label{sec:dim_est_method}

Our dimensionality estimator is based on the ratio of nearest-neighbor distances. Notably, {\name}   only requires mean-value estimates, avoiding the need for explicitly knowing the distribution of the nearest-neighbor ratio.

Suppose that a finite set $\cX\subset \R^D$ is sampled from a $d$-dimensional subspace, with $d<D$. Our goal is to estimate $d$ from $\cX$. 
For $x\in\R^D$ and $k\ge 1$, we define the $k^{th}$ nearest neighbor distance of $x$ with respect to $\cX$ as:
\[    
R_k(x,\cX) := \inf
\set{r > 0: {|B_r(x)\cap \cX \bs \set{x}| \ge k}},
\]  
where $B_r(x)$ is the $D$-dimensional Euclidean ball of radius $r$ around $x$.
For $k>j\ge1$ we define,
\begin{equation}\label{eqn:Lkj}
 L_{k,j}(x,\cX) := -\logg\param{\frac{R_k(x,\cX)}{R_j(x,\cX)}},
\end{equation}
and its average,
\begin{equation}\label{eqn:Lkj_avg}
\bar L_{k,j}(\cX) = \frac1{|\cX|}\sum_{x\in\cX} L_{k,j}(x,\cX).
\end{equation}
 Theorem \ref{thm:univ_limit} below shows that asymptotically, $\bar L_{k,j}$ linearly depends on $\log(d)$ (see Remark \ref{rem:coeffs} for more details). Consequently,  
our proposed {\name} estimator for $d$ based on the $(k,j)$--nearest neighbor distances is
\begin{equation}
\widehat{d}_{k,j}(\cX) = \exp\param{\alpha_{k,j}\bar{L}_{k,j}(\cX)+\beta_{k,j}},
\label{eqn:dhat}
\end{equation}
where $\alpha_{k,j}$ and $\beta_{k,j}$ are pre-determined constants. Importantly, while $\alpha_{k,j}$ and $\beta_{k,j}$ might depend on the size of $\cX$, they  are otherwise \emph{independent of any other property of the sample $\cX$} or its distribution.  
This fact allows us to perform a simple tuning stage for our estimator, to optimize the values of $\alpha_{k,j}$ and $\beta_{k,j}$. We use least-squares to fit the slope and intercept for $\bar L_{k,j}(\cX)$ as a function of $\log(d)$, using samples generated by the $d$-dimensional normal distribution; see Fig.\,\ref{fig:mapping_curves}. We emphasize that the tuning stage is only performed once (for a given sample size), and the values of $\alpha_{k,j}$ and $\beta_{k,j}$ can be stored and reused in future applications, regardless of the type of the data (see supplementary materials for a table of precomputed values).

\begin{remark}\label{rem:coeffs}
    Theorem \ref{thm:univ_limit} below states that asymptotically $$\bar L_{k,j}(\cX) \approx \log(d) + C_{k,j},$$ implying that we should take $\alpha_{k,j}\!=\!1, \beta_{k,j}\!=\!C_{k,j}$. In practice, however, we observed that 
   while this holds for a large sample size $n$ (as the limit implies), for small sample sizes  we have $$L_{k,j}(\cX) \approx \alpha_{k,j}^{(n)}\log(d) +\beta_{k,j}^{(n)}, $$ where $\alpha_{k,j}^{(n)} < 1$. Therefore, to get the best results, we first need to learn the values of $\alpha_{k,j}^{(n)}$ and $\beta_{k,j}^{(n)}$, suitable for the given sample size, as described above. The details on how we estimate these values are given in Section \ref{sec:dim_map_construction}.
\end{remark}

\begin{figure}
    \centering
    \begin{tabular}{c}
    \includegraphics[width=0.99\linewidth]{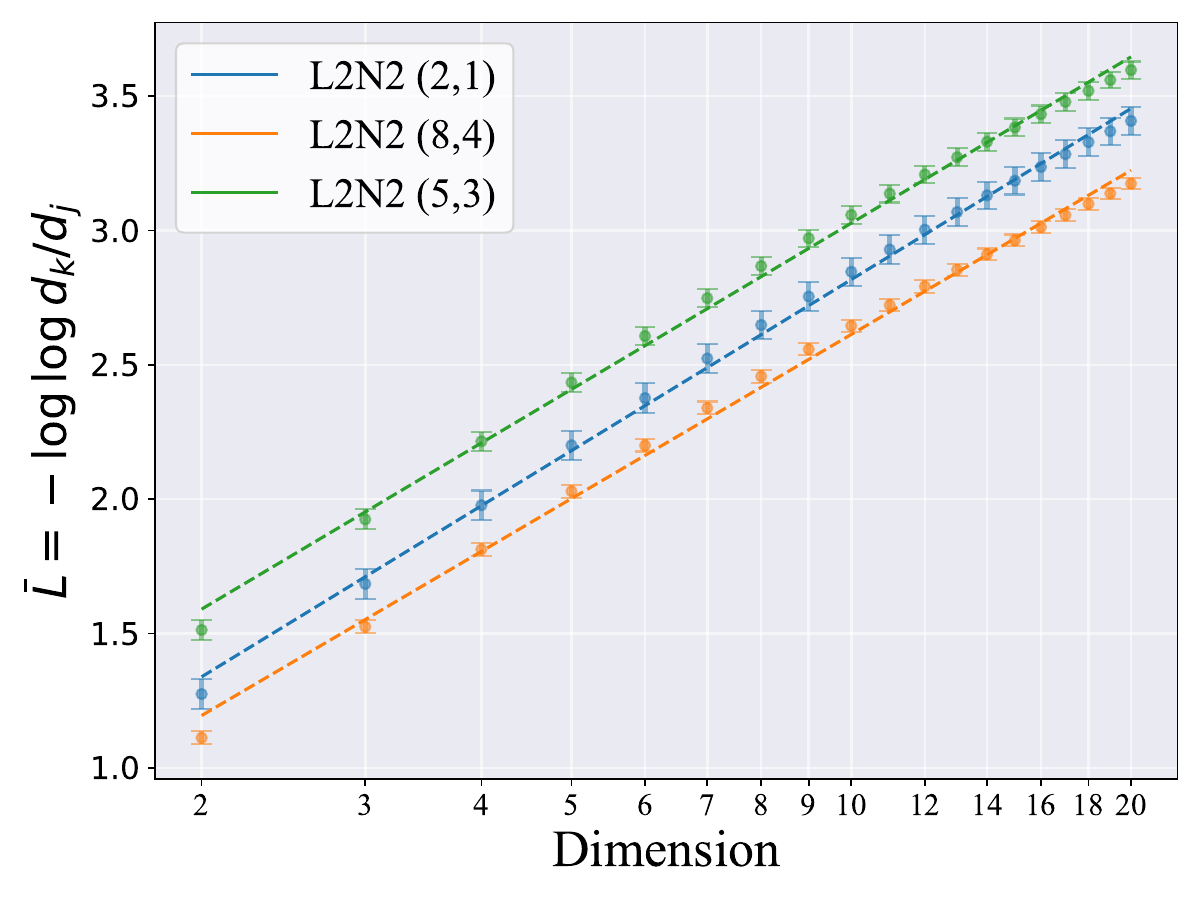} \\
    \end{tabular}
    \caption{%
      The (approximately) linear relationship between $\bar L_{k,j}$ and $\log(d)$ (for $(k,j) = (2,1)$, $(5,3)$, and $(8,4)$). Here we used  2,500 points, sampled from the $d$-dimensional Gaussian distribution. The error bars are computed from 1000 independent trials. 
    } 
    \label{fig:mapping_curves}
\end{figure}

\section{Theoretical Analysis}
\label{sec:theory}
In this section 
we provide the theoretical analysis to support our estimator $\widehat{d}_{k,j}(\cX)$ \eqref{eqn:dhat}, and to show that it is a consistent estimator for $d$, independently of the data distribution or its support.
Let $\cM\subset \R^D$ be a $d$-dimensional $C^1$ manifold embedded in $\R^D$. We assume that $f:\cM\to\R$ is a probability density function on $\cM$, such that
\[
    0 < \fmin = \inf_{\cM} f \le \sup_{\cM} f = \fmax < \infty.
\]

Let $\cX_n = \set{X_1,\ldots, X_n}$ be a set of independent random variables with density $f$. We are interested in the limit of $\bar L_{k,j}(\cX_n)$ \eqref{eqn:Lkj_avg} as $n\to\infty$. We use  $Y_n \xrightarrow{P} Y$  to denote convergence in probability, i.e. $\mathbb{P}(|Y_n-Y| \ge \eps)\to 0$ for any $\eps>0$ as $n \rightarrow \infty$.

\begin{thm}\label{thm:univ_limit}
Under the assumptions above, as $n \rightarrow \infty$,
\[
\bar L_{k,j}(\cX_n) \xrightarrow{P} \log(d) + C_{k,j},
\]
where $C_{k,j}$ is independent of $d$ and $f$, and is given in \eqref{eqn:Ckj}.
\end{thm}

\begin{remark} 
We make a few comments about 
Theorem \ref{thm:univ_limit}:

\noindent (1) The theorem shows a universal limit for the average $L_{k,j}$ value, which is sufficient for justifying the use $\widehat d_{k,j}(\cX)$. However, using the framework  developed in \cite{bobrowski2024universalityrandompersistenthomology}, one can prove that, in fact, the \emph{empirical distribution} of the collection of values 
$\{L_{k,j}(X_i,\cX_n)\}_{i=1,\ldots, n}$ 
also admits a \emph{universal limit}. 
Applications and a proof of this statement are left for future work.
 
 \noindent 
 (2) Theorem \ref{thm:univ_limit} is limited to $C^1$ manifolds. However,  the space of $C^1$ manifolds is quite general and covers many useful cases. Additionally, the good  performance of the {\name} estimator in practice suggests that the theorem 
 holds in much greater generality. There are several significant challenges in generalizing the proofs to more general spaces, and so we will address this in future work.

 \noindent (3) 
To prove Theorem \ref{thm:univ_limit},
we use the framework introduced in \cite{Penrose2011LimitTF},  for limit theorems of point processes on manifolds. It
was  originally used to show that the Levina-Bickel MLE estimator \cite{levina} is universal in the same sense as ours. As we show in Section~\ref{sec:res}, while the MLE estimator also asymptotically converges to $d$,  {\name} performs significantly better in practice. We discuss this further in Section \ref{sec:exp_discussion}.
\end{remark}

In the proof of Theorem \ref{thm:univ_limit}, we will make use of a truncated version of $\bar L_{k,j}$,  ignoring points with a far-away $k$-nearest neighbor, making the estimator easier to control. For $\cX_n = \set{X_1,\ldots, X_n}$,  we set 
   \[\bar L_{k,j,r}(\cX_n) := \frac1n\sum_{i=1}^n L_{k,j,r}(X_i,\cX_n), 
\]
where
\begin{equation}\label{eqn:Lkjr}
    L_{k,j,r}(x,\cX_n) := L_{k,j}(x,\cX_n) \cdot \indf{R_k(x,\cX_n) \le r},
\end{equation}
   $\indf{\cdot}$ is the indicator function, and $r>0$ is a constant. 
The proof of Theorem \ref{thm:univ_limit} splits into three main parts:

\vspace{5pt}

\noindent{\bf Step 1:} Show  that the limit for the truncated average $\bar L_{k,j,r}$ can be expressed by means of a  mixture of homogeneous Poisson processes:
\begin{lem}\label{lem:L2_limit}
    For a fixed $r>0$, as $n \rightarrow \infty,$
    \[
    \bar L_{k,j,r}(\cX_n) \xrightarrow{P} \int_{\cM}\mean{L_{k,j}(\mathbf{0}, \cH_{f(x)})}f(x)dx,
    \]
    where $\cH_{\lambda}$ is a homogeneous Poisson process  
    in $\R^d$ with rate $\lambda$, and $\mathbf{0}$ is the origin in $\R^d$.
\end{lem}

Intuitively, the integral in Lemma \ref{lem:L2_limit}
says that, in the limit, we have points covering the manifold $\cM$, and the neighborhood of each point looks like a homogeneous Poisson process with rate $f(x)$.

\vspace{5pt}

\noindent{\bf Step 2:} Show that the loss of the truncated version is negligible, so that $\bar L_{k,j,r} \approx \bar L_{k,j}$:
\begin{lem}\label{lem:bounded_limit}
$\limninf \prob{\bar L_{k,j,r}(\cX_n) \ne \bar L_{k,j}(\cX_n)} = 0$.
\end{lem} 

\noindent{\bf Step 3:} Compute the expected value for a homogenous Poisson process around the origin:
\begin{lem}\label{lem:homogeneous}
Let $\cH_\lambda$ be a homogeneous Poisson process in $\R^d$ with rate $\lambda$, then
\[
\mean{L_{k,j}(\mathbf{0},\cH_\lambda)} = \log(d) + C_{k,j},
\]
where $C_{k,j}$ is given in \eqref{eqn:Ckj}, and is independent of $d,\lambda$.
\end{lem}

The proof for Lemma \ref{lem:homogeneous} reveals the connection between $\bar L_{k,j}$ and the intrinsic dimension $d$ and so is given below.
The proofs for Lemmas \ref{lem:L2_limit} and \ref{lem:bounded_limit} are more technical, mainly consisting of bounding moments.
Using  Lemmas \ref{lem:L2_limit}-\ref{lem:homogeneous}, we can now prove our main theorem.

\subsection{Proofs}
\begin{proof}[Proof of 
Theorem \ref{thm:univ_limit}]
Set,
\[
L := \int_{\cM}\mean{L_{k,j}(\mathbf{0}, \cH_{f(x)})}f(x)dx.
\]
From Lemmas  \ref{lem:L2_limit} and \ref{lem:bounded_limit}  we have that for any $\epsilon >0$,
\begin{align*}
\mathbb{P}(|\bar L_{k,j}(\cX_n) &- L| > \epsilon )\le\mathbb{P}(| \bar L_{k,j, r}(\cX_n) - L| > \eps/2 )  \\
& + \mathbb{P}(|\bar L_{k,j}(\cX_n) - \bar L_{k,j, r}(\cX_n)| > \eps/2 )
\rightarrow 0,
\end{align*}
hence $
    \bar L_{k,j}(\cX_n)\xrightarrow{P} L$, as $n \rightarrow \infty$.
From Lemma \ref{lem:homogeneous} it follows that $L = \log(d)+C_{k,j}$.
\end{proof}
\begin{proof}[Proof of Lemma \ref{lem:L2_limit}]
   We will use Theorem 3.1 from \cite{Penrose2011LimitTF}, which proves the result for translation and rotation invariant functionals which satisfy a moment condition.
 From the definitions of $L_{k,j}$ \eqref{eqn:Lkj} and $L_{k,j,r}$ \eqref{eqn:Lkjr}, it is clear
    the $L_{k,j}$ functional is both translation and rotation invariant, as required in \cite{Penrose2011LimitTF}. 
   So for Theorem 3.1 in \cite{Penrose2011LimitTF} to hold we need to show that
    \[
\sup_{n}\mean{\abs{L_{k,j,r}(X_1,\cX_n)}^p} < \infty,
\]
for some $p>2$. 
Conditioning on both $X_1$ and $R_k(X_1,\cX_n)$,
\[
\begin{split}
&\mathbb{E}\{|L_{k,j,r}(X_1,\cX_n)|^p\}  =\!  \int_0^\infty \prob{\abs{L_{k,j,r}(X_1,\cX_n)}^p >t}dt\\
&\quad=\int_0^\infty\int_{\cM}\int_0^r\cprob{|L_{k,j}| >t^{1/p}}{X_1 = x, R_k=u} \\
&\qquad \qquad\qquad\quad\times f_{R_k}(u | x) f(x)du\ dx\ dt,
\end{split}
\]
where $L_{k,j} = L_{k,j}(X_1,\cX_n)$, $R_k = R_k(X_1,\cX_n)$, and
 $f_{R_k}(u | x)$ is the conditional density of  
 $R_k$ given $X_1=x$.
 Next, define
\begin{align*}
    I_1(t,u) &:= \cprob{L_{k,j} <-t^{1/p}}{X_1=x, R_k=u}\\ 
    I_2(t,u) &:= \cprob{L_{k,j} >t^{1/p}}{X_1=x, R_k=u}.
\end{align*}
We will bound each term separately.
For $I_1$, note that $L_{k,j}<-t^{1/p}$ is equivalent to $R_k/R_j > v_1^{-1}$, where $v_1 = v_1(t) := e^{-e^{t^{1/p}}}$.
Therefore, 
\[
I_1(t,u) = \cprob{R_j < v_1 R_k}{X_1=x, R_k =u}.
\]
When $R_k=u$, we know that $B_u(x)$ contains exactly $(k-1)$ points from $\cX_{n-1}$. In addition, if $R_j < v_1 R_k$ then the ball $B_{v_1\cdot u}(x)$ contains at least $j$ points out of these $(k-1)$ points. Therefore, denoting $F(A) = \int_A f(z)dz$,
\[
I_1(t,u) \le \binom{k-1}{j}\param{\frac{F(B_{v_1u}(x))}{F(B_u(x))}}^j \le C_1 v_1^{jd} < C_1 v_1, 
\]
where $C_1>0$ is a constant that does not depend on $x,t,u$, and we used Lemma 4.3 in \cite{Penrose2011LimitTF}. Since our bound is independent of $u$ and $x$, we have
\begin{align}\label{eqn:I_1}
\int_0^\infty \int_{\cM}\int_0^r &I_1(t,u)f_{R_k(u)}f(x)du dx dt \nonumber\\ &\le C_1r\int_0^\infty e^{-e^{t^{1/p}}}dt.
\end{align}

For $I_2$, note that $L_{k,j} >t^{1/p}$ is equivalent to $R_k/R_j < v_2^{-1}$ where $v_2 = v_2(t):= e^{-e^{-t^{1/p}}}$. Similarly to the above argument, we have that if $R_j > v_2R_k$ then at least $(k-j)$ points in $B_u(x) \bs B_{v_2u}(x)$, and therefore,
\begin{align*}
I_2(t,u) & \le \binom{k-1}{k-j}\param{\frac{F(B_u(x))-F(B_{v_1u}(x))}{F(B_u(x))}}^{k-j} \\
& \le  C_2 (1-v_2^{d})^{k-j} < C_2 (1-v_2^d), 
\end{align*}
implying that

\begin{align}
\int_0^\infty \int_{\cM}
\int_0^r &I_2(t,u) f_{R_k}(u)f(x)du  dxdt \nonumber\\
&\le C_2r\int_0^\infty (1-e^{-e^{-t^{1/p}}})dt.
\label{eqn:I_2}
\end{align}

Since the right-hand-side on both \eqref{eqn:I_1} and \eqref{eqn:I_2} is finite for $p=3$, the proof is complete.

\end{proof}
\begin{proof}[Proof of Lemma \ref{lem:bounded_limit}]
    Note that if $\bar L_{k,j,r} \ne \bar L_{k,j}$ then there exists $i$ for which $R_k(X_i,\cX_n) >r$. Therefore,
    \begin{align*}
        &\prob{\bar L_{k,j,r}(\cX_n)\ne \bar L_{k,j}(\cX_n)} \le n \prob{R_k(X_1,\cX_n) > r} \\
        &\quad = n\int_\cM \cprob{R_k(X_1,\cX_n)>r}{X_1=x}f(x)dx.
    \end{align*}
    Given $X_1=x$, let $N_x$ be the number among the $(n-1)$ points $\{X_2,\ldots,X_n\}$ which lie in $B_r(x)$. Then $N_x\sim \mathrm{Binomial}(n-1, F(B_r(x)))$. Denote $\mu_x = \mean{N_x}$. Note that from Lemma 4.3 in \cite{Penrose2011LimitTF}, we have
    \begin{align*}
    	\mu_x &= (n-1)F(B_r(x)) \ge C_\mu n r^d, 
    \end{align*}
    for some $C_\mu> 0$, independent of $x$.
     Applying a Chernoff-type bound for the binomial distribution (e.g., Lemma 1.1 in \cite{penrose2003random}), and assuming $n$ sufficiently large so that $\mu_x > 2(k-1)$, then
\[\cprob{R_k(X_1,\cX_n)\!>\!r}{X_1\!=\!x} 
=\prob{N_x \!\le\! k-1} \le e^{- C \mu_x}.
\]
Thus, we conclude that $\prob{\bar L_{k,j,r}(\cX_n)\ne \bar L_{k,j}(\cX_n)} \to 0$, completing the proof.
\end{proof}

\begin{proof}[Proof of Lemma \ref{lem:homogeneous}]
Denote $R_k = R_k(\mathbf{0}, \cH_\lambda)$, and 
$Y_k = \lambda\omega_d R_k^d$, where $\omega_d$ is the volume of a $d$-dimensional unit ball.
Note that the sequence $Y_1,Y_2,\ldots$, can be viewed as the arrival times of a unit-rate Poisson process, and therefore can be written as
$
Y_k = \sum_{i=1}^k Z_i
$,
where $Z_1,Z_2,\ldots$, are independent with $Z_i\sim\mathrm{Exp}(1)$. 
Fix $1\le j < k$, and 
set $\Delta_{k,j} = Y_k-Y_j$. Then $Y_j\sim \mathrm{Gamma}(j)$, and $\Delta_{k,j} \sim \mathrm{Gamma}(k-j)$ are independent, and therefore,
\[
    U_{k,j} := \frac{Y_j}{Y_k} = \frac{Y_j}{Y_j+\Delta_{k,j}}\sim\mathrm{Beta}(j,k-j).
\]
Denoting $L_{k,j} = L_{k,j}(\mathbf{0},\cH_\lambda)$, then by \eqref{eqn:Lkj},
\[
L_{k,j} = \logg\param{\param{\frac{Y_k}{Y_j}}^{1/d}} = \log(d) -\logg(U_{k,j}^{-1})
\]
Taking $C_{k,j} = -\mathbb{E}\{\log\log(U_{k,j}^{-1})\}$, then
\begin{equation}
\begin{split}
&C_{k,j} = -\binom{k-1}{j-1}\int_0^1 \log(-\log(u)) u^{j-1}(1-u)^{k-j-1}du
\\
&= -\binom{k-1}{j-1}\int_0^\infty \log(s) e^{-js}(1-e^{-s})^{k-j-1}ds 
\\
&= -\binom{k-1}{j-1}\!\!\sum_{i=0}^{k-j-1}\!\!(-1)^i\binom{k\!-\!j\!-\!1}{i}\!\!\int_0^\infty\!\!\!\!\! \log(s) e^{-(i+j)s} ds,\\
&= \binom{k-1}{j-1}\!\!\sum_{i=0}^{k-j-1}\!\!(-1)^i\binom{k\!-\!j\!-\!1}{i}\!\frac{\gamma \!+\! \log (i+j)}{i+j}.
\label{eqn:Ckj}
\end{split}\raisetag{20pt}
\end{equation}
The last line follows from $\int_0^\infty\log(s) e^{-(i+j)s}ds = \frac{-\gamma - \log (i+j)}{i+j}$, where $\gamma$ is the Euler-Mascheroni constant  (see 
\gr{Appendix \ref{app:eval}} for the calculation). As a special case, we have $C_{2,1} = \gamma \approx 0.57721$. \end{proof}

\section{Experimental Setup}
\label{sec:exp}
In this section we provide the details for the experiments performed to evaluate the {\name} estimator, and compare it against existing dimensionality estimation methods.

\subsection{Tuning the {\name} Estimator Parameters}
\label{sec:dim_map_construction}
We describe here the \emph{parameter tuning stage} for the {\name} estimator, where we optimize the values of $\alpha_{k,j}$ and $\beta_{k,j}$ in \eqref{eqn:dhat} to adjust for finite sample effects. 
We stress that the estimation is only needed to be carried out \emph{once} for a given choice of $(k,j)$ values and sample size $n$, and can be reused for any given dataset.

For each of choice of values of $n$ and  $(k,j)$, we generated 1,000 point clouds from a $d$-dimensional Gaussian distribution for various values of $d$. Using different distributions may affect the calculated parameter values. The Gaussian distribution was chosen as it consistently yielded the best results across different settings.
This 
approach can be conceived as equivalent to  a grid-search method, and does not impose any prior distribution on the data. 
For each point-cloud, we computed the value of $\bar L_{k,j}$ \eqref{eqn:Lkj_avg}.
Following the discussion in Section \ref{sec:dim_est_method}, we use least-squares regression to find the best fit so that $\bar L_{k,j} \approx \alpha_{k,j}\log(d) +\beta_{k,j}$.
We note that an important part of this tuning stage is the range of dimensions used to fit the values of $\alpha_{k,j},\beta_{k,j}$. This stems from the fact that convergence rates become slower in higher dimensions.
In practice, if there is an approximate range of candidate values, this can be used to improve accuracy. For example, for the benchmark experiments we used $1\le d\le 20$, while for the real-world examples (except for ISOMAP), we used $10\le d \le 40$. 

\subsection{Evaluation Datasets}
\label{sec:datasets}
We tested the {\name} estimator \eqref{eqn:dhat} on three datasets: (1) benchmark manifolds, proposed by Campadelli et al.~\cite{Campadelli}, where the ground truth ID is known; (2) sampling from $d$-dimensional spheres with Gaussian noise; (3) real-world datasets with much higher ambient space dimensionalities, but with only conjectured ID's. 

\subsubsection{Benchmark Manifolds}
We evaluated {\name} on  
a collection of synthetic datasets which is widely used in the ID literature \cite{Campadelli,Facco,denti,binnie2025surveydimensionestimationmethods}. Each manifold has a known ID, embedding (ambient) dimension, and geometric/topological properties that make it suitable for benchmarking different estimators. 
The dataset consists of 24 different manifolds, with ID's ranging from 1 to 70. The embedding dimensions are between 3 and 72.
The geometry of these manifolds range from flat to various degrees of curvature. 
Table \ref{tab:benchmark_manifolds} gives the ambient dimension and underlying ID values for all the benchmark manifolds from \cite{Campadelli} used in our experiments. We refer the reader to \cite{Campadelli} for full details of the geometrical nature of these manifolds.

\begin{table}[h!]
\centering
\begin{tabular}{l r r}
\hline
\textbf{Name} & \textbf{Intrinsic dim.\ ($d$)} & \textbf{Ambient dim.} \\
\hline
\hline
M1\_Sphere & 10 & 11 \\
M2\_Affine\_3to5 & 3 & 5 \\
M3\_Nonlinear\_4to6 & 4 & 6 \\
M4\_Nonlinear & 4 & 8 \\
M5a\_Helix1d & 1 & 3 \\
M5b\_Helix2d & 2 & 3 \\
M6\_Nonlinear & 6 & 36 \\
M7\_Roll & 2 & 3 \\
M8\_Nonlinear & 12 & 72 \\
M9\_Affine & 20 & 20 \\
M10a\_Cubic & 10 & 11 \\
M10b\_Cubic & 17 & 18 \\
M10c\_Cubic & 24 & 25 \\
M10d\_Cubic & 70 & 71 \\
M11\_Moebius & 2 & 3 \\
M12\_Norm & 20 & 20 \\
M13a\_Scurve & 2 & 3 \\
M13b\_Spiral & 1 & 13 \\
Mbeta & 10 & 40 \\
Mn1\_Nonlinear & 18 & 72 \\
Mn2\_Nonlinear & 24 & 96 \\
Mp1\_Paraboloid & 3 & 12 \\
Mp2\_Paraboloid & 6 & 21 \\
Mp3\_Paraboloid & 9 & 30 \\
\hline
\end{tabular}
\caption{List of benchmark manifolds \cite{Campadelli}}
\label{tab:benchmark_manifolds}
\end{table}


For testing the estimator performance, we use Mean Percentage Error (MPE) as the metric.  Given an estimate $\widehat d$ and its corresponding ground truth $d$, the MPE is given by
\[
\mathrm{MPE} = 100\times \frac{|\widehat d - d|}{d}. 
\]
For each manifold, we evaluated the performance on point-clouds of 625, 1,250, 2,500 and 5,000 points, with 20 repetitions each. 
We compared 
{\name} with 14 existing methods 
(see Table~\ref{tab:mpe_benchmark}). For methods with tuneable hyper-parameters, we used the values from \cite{binnie2025surveydimensionestimationmethods}, which optimize the average MPE across all manifolds. 
We evaluated the performance of $\widehat d_{k,j}$ for $(k,j) = (2,1), (4,2),(8,4)$.

Finally, as a ``sanity check", in addition to the benchmark manifolds in \cite{binnie2025surveydimensionestimationmethods}, we tested {\name} on 2,500 point samples from $d$-dimensional spheres, for $10\le d\le 40$.
\subsubsection{Noisy Datasets}\label{subsec:noisy}
In the original benchmark, the data lies     completely on a low-dimensional manifold. Beyond the benchmark, we tested the effect of sampling noise (in the ambient space) on our estimator.
 We used the 6-dimensional sphere $S^6$ embedded in $\mathbb{R}^{11}$, with point cloud of 2,500 points. We added mean zero  Gaussian noise to each
 point, 
    with standard deviations ($\sigma$) 
    of 0.01, 0.1 and 1.0. We performed 500 repetitions per $\sigma$ value. We ran the same experiment for $S^{10}$ embedded in $\mathbb{R}^{11}$. These experimental configurations allow us to compare against other methods using results from \cite{binnie2025surveydimensionestimationmethods} (see supplementary materials).

\subsubsection{Real-World Datasets}
We  evaluated {\name} on the following real-world datasets: the ISOMAP face dataset \cite{ISOMAP}, the  MNIST dataset \cite{mnist}, the CIFAR-100 dataset \cite{CIFAR100} and the Isolet dataset \cite{isolet_54}. 

The ISOMAP face dataset consists of 698 grayscale 64x64 images of a sculpture face obtained from different pose angles and lighting directions, so its ID is 3. 

The MNIST dataset contains 70,000 greyscale images of 28x28 hand-written digits and its true ID  is \emph{not known}. However, as written digits are highly structured, the ID should be substantially smaller than 784. 

The CIFAR-100 dataset \cite{CIFAR100} consists of 60,000 color images, each of resolution 32x32 pixels, resulting in a flattened vector of 3,072 dimensions. 

Finally, the Isolet dataset \cite{isolet_54} consists of features extracted from audio recordings of 150 subjects, pronouncing each letter of the alphabet twice, with 59 examples per speaker, totaling  7,797 examples with each example having 617 features.

To test the effect of sample size, we took 20 random subsets of the point clouds over a range of different subset sizes.

\subsection{Downstream Experiments}
To further evaluate the accuracy of the L2N2 estimates, we conducted a downstream experiment using autoencoders on the MNIST dataset, where no ground truth is available. By testing different numbers of hidden units in the bottleneck layer, we checked for which choice the reconstruction error reaches an optimum, i.e. the smallest number of hidden units which result in a small error. The reasoning being that this should correspond to the intrinsic dimension.  
We used images corresponding to the individual digits. We created an autoencoder model that consisted of 6 layers with the following number of hidden units: $256 \rightarrow 128 \rightarrow B \rightarrow 128 \rightarrow 256 \rightarrow 768$, where $B$ indicates the number of bottleneck hidden units. Mean squared error was used as the loss function for training and evaluation. The values of $B$ were chosen in the range of 13 and 31, since it contains the intrinsic dimensionalities for differet MNIST digits that the methods of TwoNN, GriDE, MLE and L2N2 converge to when all available data is used.

\section{Experimental Results}
\label{sec:res}
In this section, we present  the results for the different experiments: benchmark manifolds (Section \ref{sec:benchmark_results}), noise experiments (Section \ref{sec:noise_res}) and real-world datasets (Section \ref{sec:realworld_results}).
\subsection{Benchmark Manifold Results}
\label{sec:benchmark_results}
The average MPE across all benchmark manifolds for all methods considered is shown in Table \ref{tab:mpe_benchmark}. The MPE values for existing approaches from CDim to WODcap were obtained from \cite{binnie2025surveydimensionestimationmethods}.
It can be seen that 
\name(2,1), i.e. $k\!=\!2,j\!=\!1$, outperforms all other methods across all sample sizes.
Note that all methods in Table \ref{tab:mpe_benchmark} use parameters  optimized for performance on the benchmark. We stress that  for \name, \emph{no such optimization was involved}, yet it yields better results. For a comparison on more equal terms, we also show the results for \name(2,1)(opt),  where  $\alpha_{2,1}$ and $\beta_{2,1}$ were optimized for the benchmark, yielding a substantial additional improvement.
Finally, if we know the space has an integer dimension, we can further increase the accuracy by rounding the estimated dimensions (\name(2,1)(r)). 

\begin{table}[]
    \centering
    \begin{tabular}{|l|r|r|r|r|}
\hline
No. Pts & 625 & 1,250 & 2,500 & 5,000 \\
\hline
\hline
\name (2,1)(r) & {\bf 8.52} & {\bf 6.97} & {\bf 6.06} & {\bf 5.52}  \\
\name(2,1)(opt) & {11.62} & {9.21} & {8.09} & {7.23}\\
\name(2,1) & {13.02} & {11.00}  & {9.96} & {9.28}\\
\name (4,2) & 16.43 & 12.59 & 11.38 & 10.86 \\
\name (8,4) & 19.33 & 17.36 & 12.96 & 11.29\\
CDim \cite{ConicalDim}    & 53.91 & 50.81 & 45.78 & 43.83  \\
CorrInt \cite{grassberger1983characterization}  & 26.32 & 26.47 & 20.99& 19.05  \\
DANCo \cite{Ceruti}   & 27.16 & 22.46 & 12.61 & 17.01  \\
ESS  \cite{ESS} & 23.34 & 16.95 & 17.15 & 14.88  \\
FisherS \cite{FisherS}  & 40.92 & 42.09 &  42.24 & 42.26  \\
GriDE \cite{denti}     & 13.86 & 11.89 & 10.98 & 10.12  \\
kNN  \cite{knn}     & 16.75 & 33.09 &  14.25 & 9.92  \\
lPCA\cite{Fan2010IntrinsicDE}       & 32.29 & 23.06 & 23.07 & 18.40  \\
MIND ML\cite{mindml}    & 13.36 & 12.90 & 20.43 & 20.28  \\
MLE \cite{levina}       & 22.27 & 20.33 & 15.96 & 11.73  \\
PH  \cite{PH}    & 18.89 & 15.97 & 12.52 & 10.69  \\
TLE \cite{TLE}       & 24.44 & 22.69  & 20.57 & 14.65  \\
TwoNN \cite{Facco}     & 13.54 & 12.05 & 10.99 & 10.21  \\
WODcap \cite{WODCAP}  & 41.88& 40.49 & 40.63 & 37.47  \\
\hline
    \end{tabular}
    \caption{MPE for benchmark manifolds. The suffix ``(r)" indicates that the estimated dimensions were rounded to the nearest integer, and ``(opt)'' indicates that the parameter values were optimized for the benchmark. Note the results reported for all methods other than {\name} were optimized as well \cite{binnie2025surveydimensionestimationmethods}.}
    \label{tab:mpe_benchmark}
\end{table}

The next two methods with the best performance are the TwoNN and GriDE, which are of a similar nature to {\name}. A more detailed picture 
can be seen in the MPE scores across the individual benchmark manifolds, shown in Fig.\,\ref{fig:indiv_manifold_mpe}\,(a). We observe that {\name} performs better on non-linear manifolds, in particular when the ID is high. An exception to this is the M8\_nonlinear dataset. When we do not round our estimated dimensions, other methods are better 
when the ID is very low, such as $d=1$ or $d=2$. Here, slight deviations from the ground truth can produce large MPEs (e.g., an estimate of $\widehat d=1.1$ when $d=1$ gives an MPE of 10). However, rounding small dimensions  almost always reveals the correct dimension, see Fig.\,\ref{fig:indiv_manifold_mpe}(b). This explains the significant improvements obtained when we use dimension rounding.

\begin{figure*}
    \centering
    \includegraphics[width=0.9\linewidth]{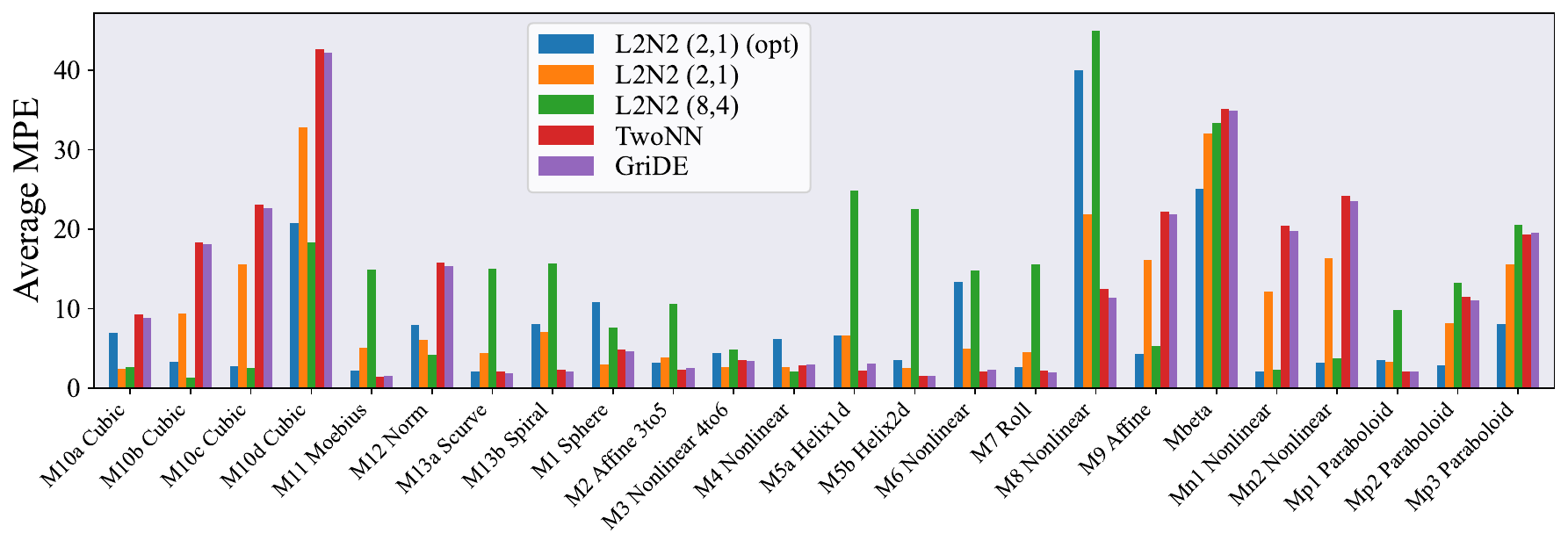} 
    \\
    \vspace{-0.5cm}
    (a) \\
    \vspace{0.3cm
    }
    \includegraphics[width=0.9\linewidth]{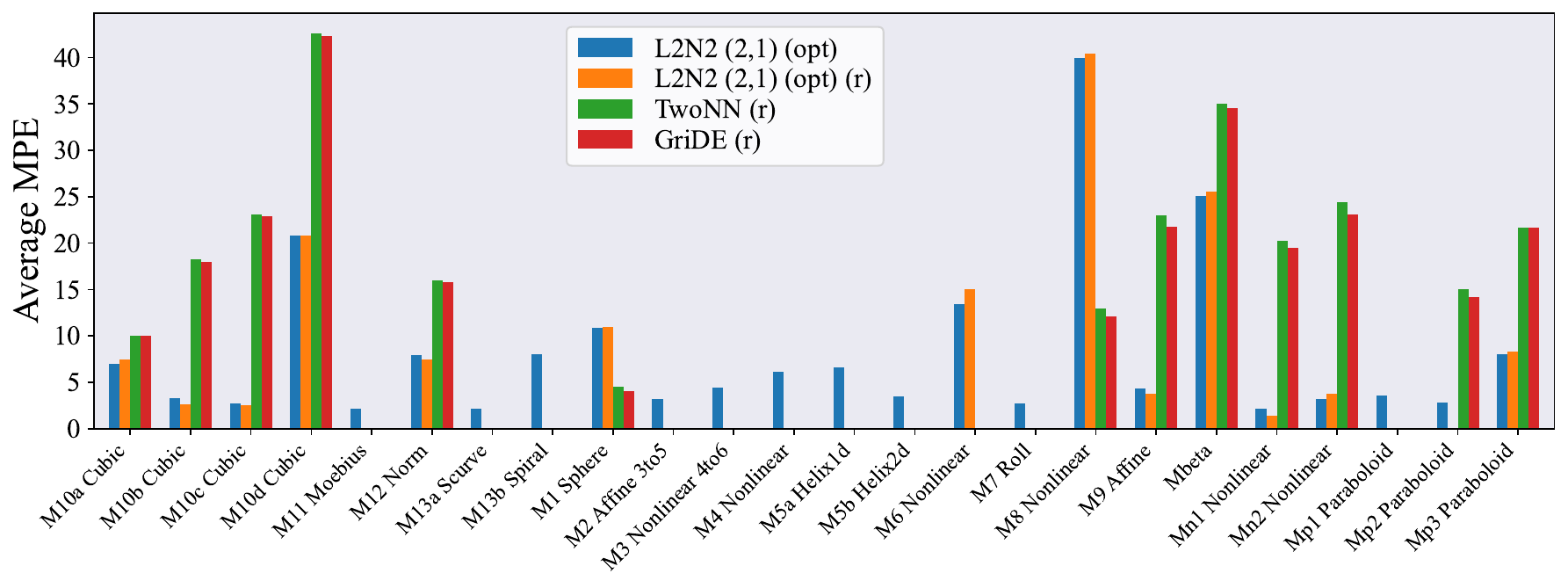} 
    \\
    \vspace{-0.5cm}
    (b) 
    \caption{(a) MPE scores of {\name} with $(k,j)=(2,j)$ and 
    $(8,4)$, TwoNN, and GriDE across the individual benchmark manifolds for 2,500 points. (b) Comparing the MPE  when dimensionality rounding is used. {\name} without rounding is here added for reference.
    }
    \label{fig:indiv_manifold_mpe}
\end{figure*}

The results for $d$-dimensional spheres are presented in Fig.\,\ref{fig:sph_comparison}.
 We see that {\name} accurately recovers the true dimensionality of the spheres, whereas TwoNN, GriDE and MLE yield systematically lower estimates, with the gap increasing with the dimensions.

 \begin{figure}
    \centering
    \includegraphics[width=.9\linewidth]{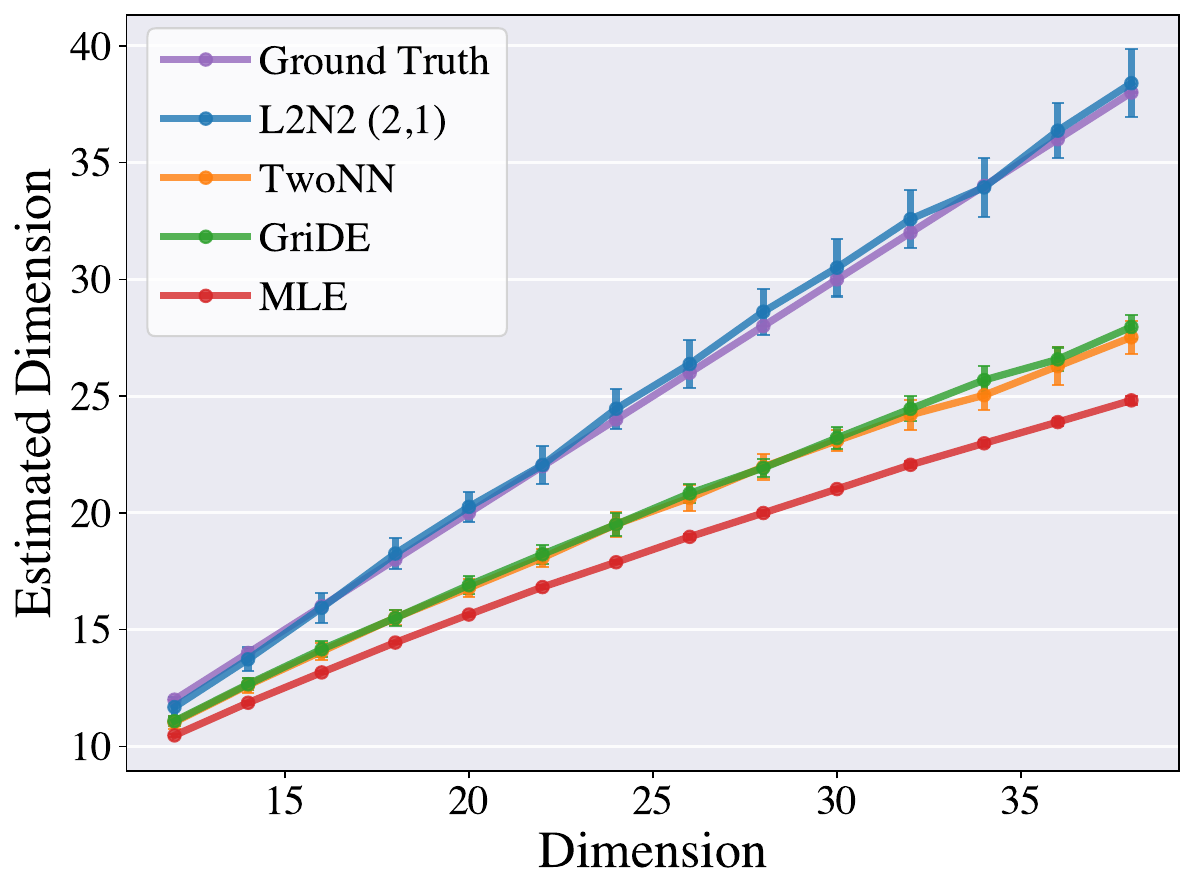}
    \caption{Comparison of estimated ID of $d$-dimensional spheres for different methods and dimensions with the ground truth shown.}
    \label{fig:sph_comparison}
\end{figure}

\subsection{Noise Experimental Results}
\label{sec:noise_res}
Table \ref{tab:noise_exp_res1} and \ref{tab:noise_exp_res2} presents the full results for the noise experiments for other dimensionality estimation approaches. It appears that all methods are sensitive to noise, and L2N2 performs competitively with the best of them.
%
In both cases, the estimated dimensions 
increase with the amount of noise, as expected. The increase in {\name} is comparable to the other methods. Generally, estimating the ID from noisy data is extremely challenging, and to some extent an ill-defined problem, as the dimension of the support of the distribution is the same as the ambient space. We leave it as future work to see if we can improve our estimates from noisy data.

\begin{table}[]
    \centering
    \begin{tabular}{|l|r|r|r|}\hline
         Method & $\sigma = 0.0$ & $\sigma = 0.01$ & $\sigma = 0.1$ \\
         \hline
        L2N2(2,1) &  6.10  & 6.16  & 8.48  \\
lPCA& 7.00 & 7.00 & 11.00  \\
MLE& 5.89 &  5.92  &8.04  \\
PH& 5.94 & 5.97  &8.21  \\
KNN& 6.00 & 5.97  &9.17  \\
WODCap& 6.93 & 6.93& 6.93  \\
GRIDE& 5.86  &8.60  &10.29 \\
TwoNN& 5.85 & 8.58 & 10.23 \\
DANCo& 6.92 & 8.95 & 11.00 \\
MiND ML& 6.00 & 9.00 & 10.00\\
CorrInt &{5.82} & 7.48 & 9.00 \\
ESS& 6.07 & 7.87 & 10.36  \\
FisherS &6.98 & { 5.63}   & {5.82} \\
TLE &6.28 & 7.83& 9.74 \\
\hline
    \end{tabular}
    \caption{Average estimated dimensions for the noise experiments where $S^6$ is embedded into $\mathbb{R}^{11}$ with mean zero, variance $\sigma^2$ Gaussian noise (in ambient space dimension) 
    added. 
    }
    \label{tab:noise_exp_res1}
\end{table}

\begin{table}[]
    \centering
    \begin{tabular}{|l|r|r|r|}\hline
    Method & $\sigma : 0.0$ & $\sigma : 0.01$ & $\sigma : 0.1$\\
         \hline
    L2N2(2,1) & 10.05 & 10.02 & 10.47  \\
        lPCA & 11.00  & 11.00 & 11.00  \\
MLE & 9.26 &  9.25 &  9.61 \\
PH & 9.38  & 9.34 &  9.78   \\
KNN & 9.91 &  9.99 &  10.29  \\
WODCap & 9.15 &  9.15 &  9.15 \\
GRIDE & 9.40 &  9.77 &  10.48  \\
TwoNN & 9.40 &  9.82 &  10.51  \\
MiND ML & 9.40  & 10.00 &  10.00 \\
CorrInt & { 9.10} &  9.32 &  9.22  \\
ESS & 10.00 &  10.27 &  10.65 \\
FisherS & 11.00 &  { 7.87} & { 5.82} \\
DANCo & 10.90 &  11.00 &  11.00  \\
TLE & 10.10 &  9.89 &  10.02 
\\
\hline
    \end{tabular}
    \caption{Average estimated dimensions for the noise experiments where $S^{10}$ is embedded into $\mathbb{R}^{11}$ with mean zero, variance $\sigma^2$ Gaussian noise (in ambient space dimension) 
    added. 
    }
    \label{tab:noise_exp_res2}
\end{table}


\subsection{Real-World Dataset Results}
\label{sec:realworld_results}
The results for the ISOMAP dataset are in Fig.\,\ref{fig:MNIST_results}(a), comparing  
{\name} with TwoNN, and GriDE.
We can see that as the sample size increases, the predicted dimensionality of both methods approach the commonly accepted ground truth of 3. When all
points are used, the {\name} estimate is closer to 3, compared to the others.

The results for the MNIST datasets for the digit ``1'' are in Fig.\,\ref{fig:MNIST_results}\,(c).
We find that {\name}  estimates  are higher than the others. However, the changes in the estimates as the number of points varies follow a similar trend across both methods. Other digits behave similarly and can be found in the supplementary materials.

The results for the CIFAR-100 dataset is shown in Fig.\,\ref{fig:MNIST_results}(d).
As with MNIST, our estimates are higher than those from other methods.

For the Isolet dataset, the results are in Fig.\,\ref{fig:MNIST_results}(b). We observe that as the number of points increases, the estimated dimensionality decreases, with similar values for {\name}, TwoNN, and GriDE (between 10-12). 

We note that for the MNIST, CIFAR-100, and Isolet datasets, the ID estimates produced by {\name} are consistently higher than those obtained from TwoNN and GriDE. However, it is known that these methods tend to underestimate intrinsic dimensionality, especially for higher dimensions. This  is clearly illustrated in Fig.\,\ref{fig:sph_comparison}, where {\name} accurately recovers the correct dimension, while TwoNN and GriDE yield systematically lower estimates. 

\begin{figure*}
    \centering
    \begin{tabular}{cccc}
   \hspace{-0.35cm} 
\includegraphics[width=0.25\linewidth]{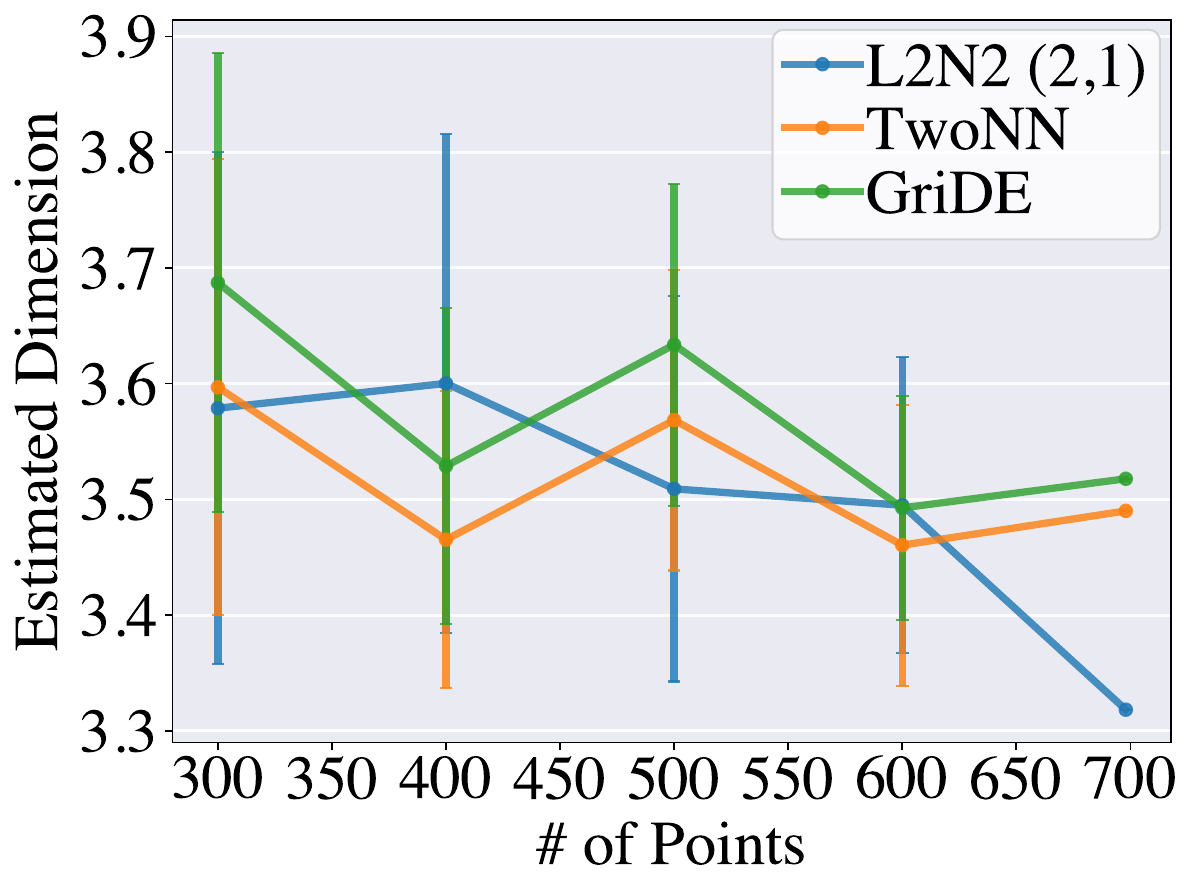} &\hspace{-0.5cm}
\includegraphics[height=94.35774pt]{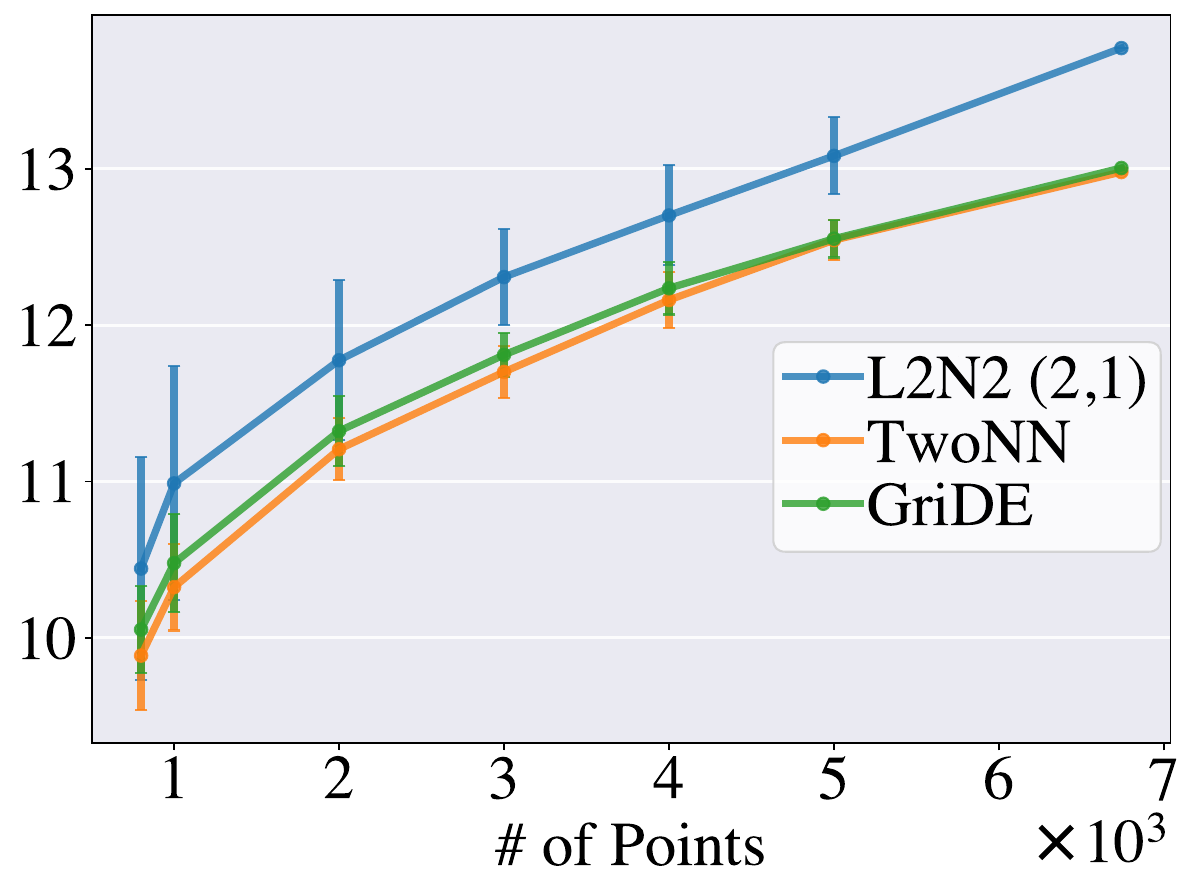} &\hspace{-0.5cm}
    \includegraphics[height=93.13774pt]{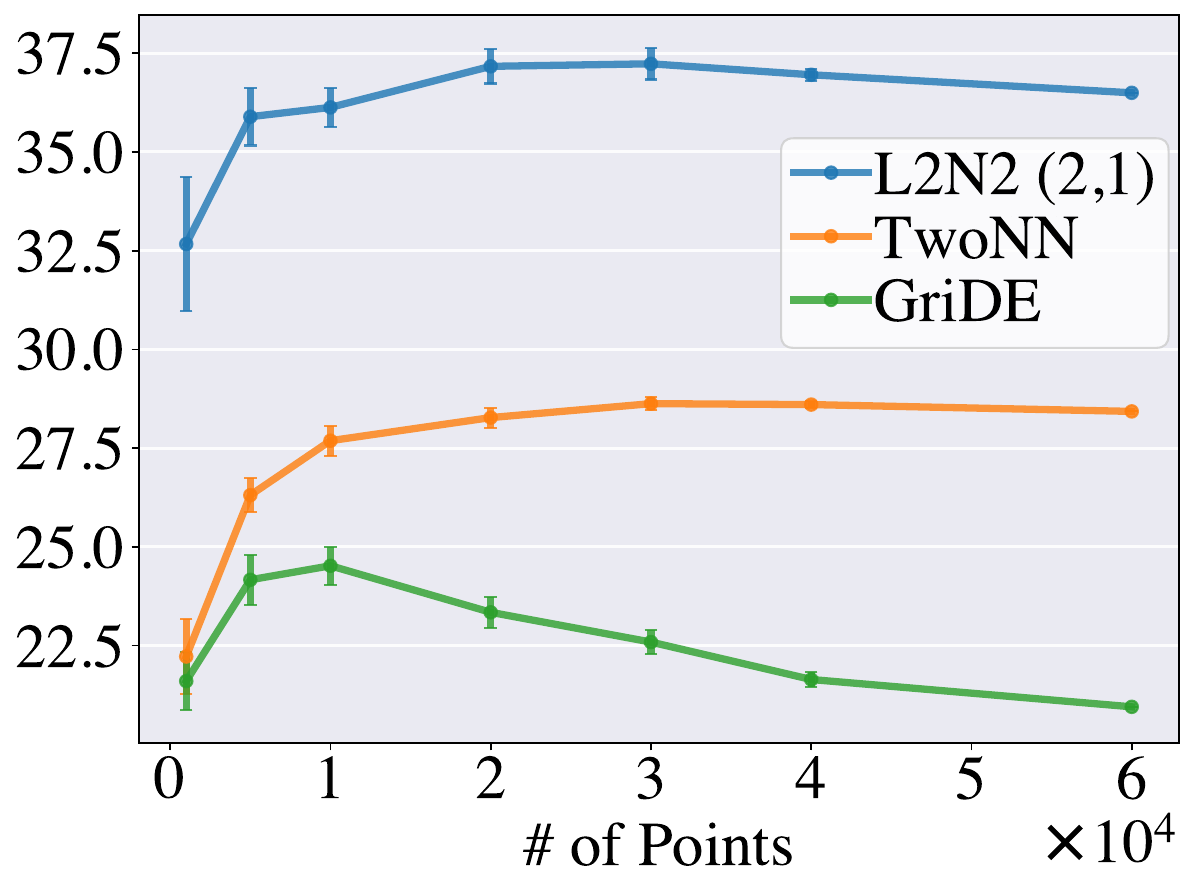} 
&\hspace{-0.5cm}
    \includegraphics[height=93.13774pt]{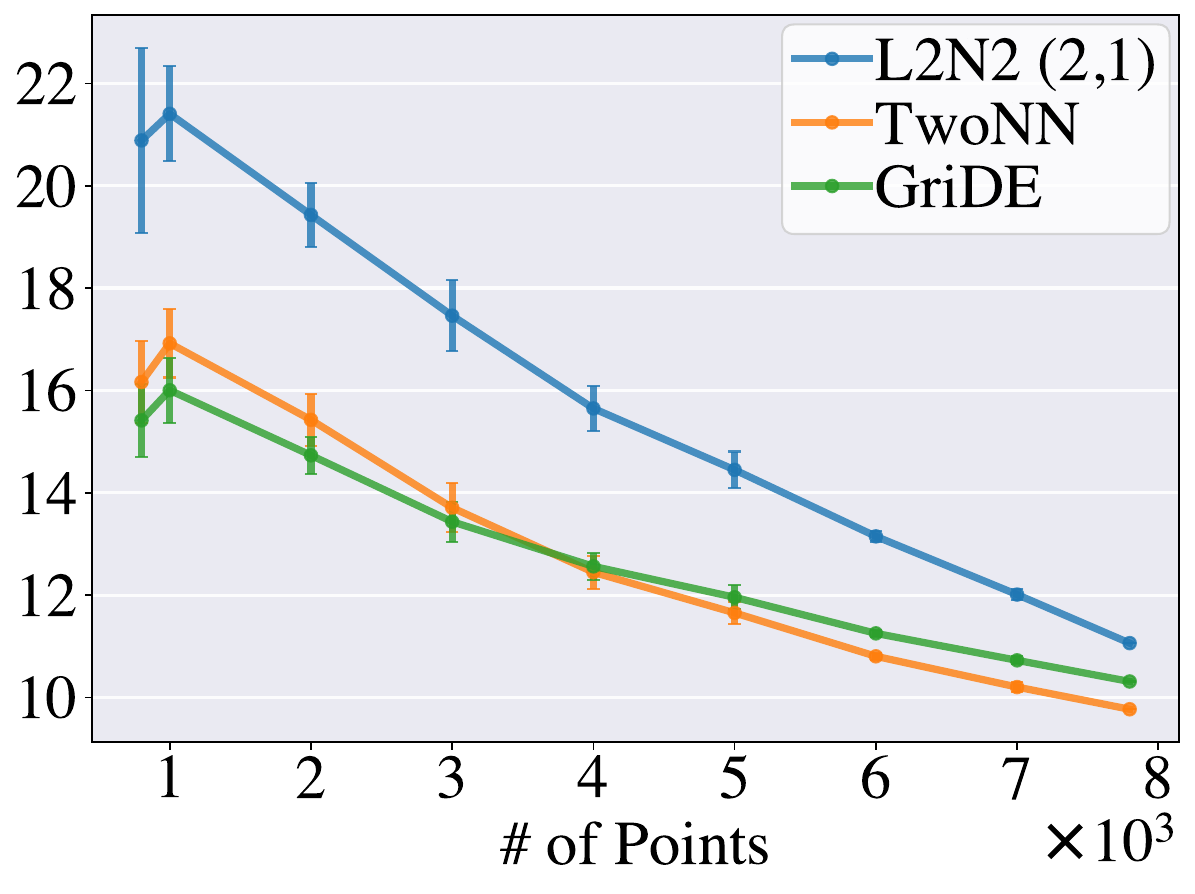} 
     \\   (a) ISOMAP & (b)  MNIST ``1"  & (c) CIFAR-100& (d)  Isolet
    \end{tabular}
    \caption{Estimated ID for real datasets. (a) ISOMAP faces;  (b)  digit ``1" in MNIST (other digits can be found in the supplemental material);   (c) CIFAR-100; (d)  Isolet dataset.  Bars denote $\pm 1$ standard deviation.}
    \label{fig:MNIST_results}
\end{figure*}

\subsection{Downstream Experiments Results}
The results can be seen in Fig. \ref{fig:autoencoder}. 
The ID estimated by TwoNN had reconstruction errors (average 0.0143) approximately 21\% higher than those obtained with L2N2 (average 0.0119)
Moreover, increasing the bottleneck beyond the L2N2 estimate does not improve performance any further, providing strong evidence that L2N2 estimates are quite accurate.

\begin{figure}
    \centering
    \includegraphics[width=\linewidth]{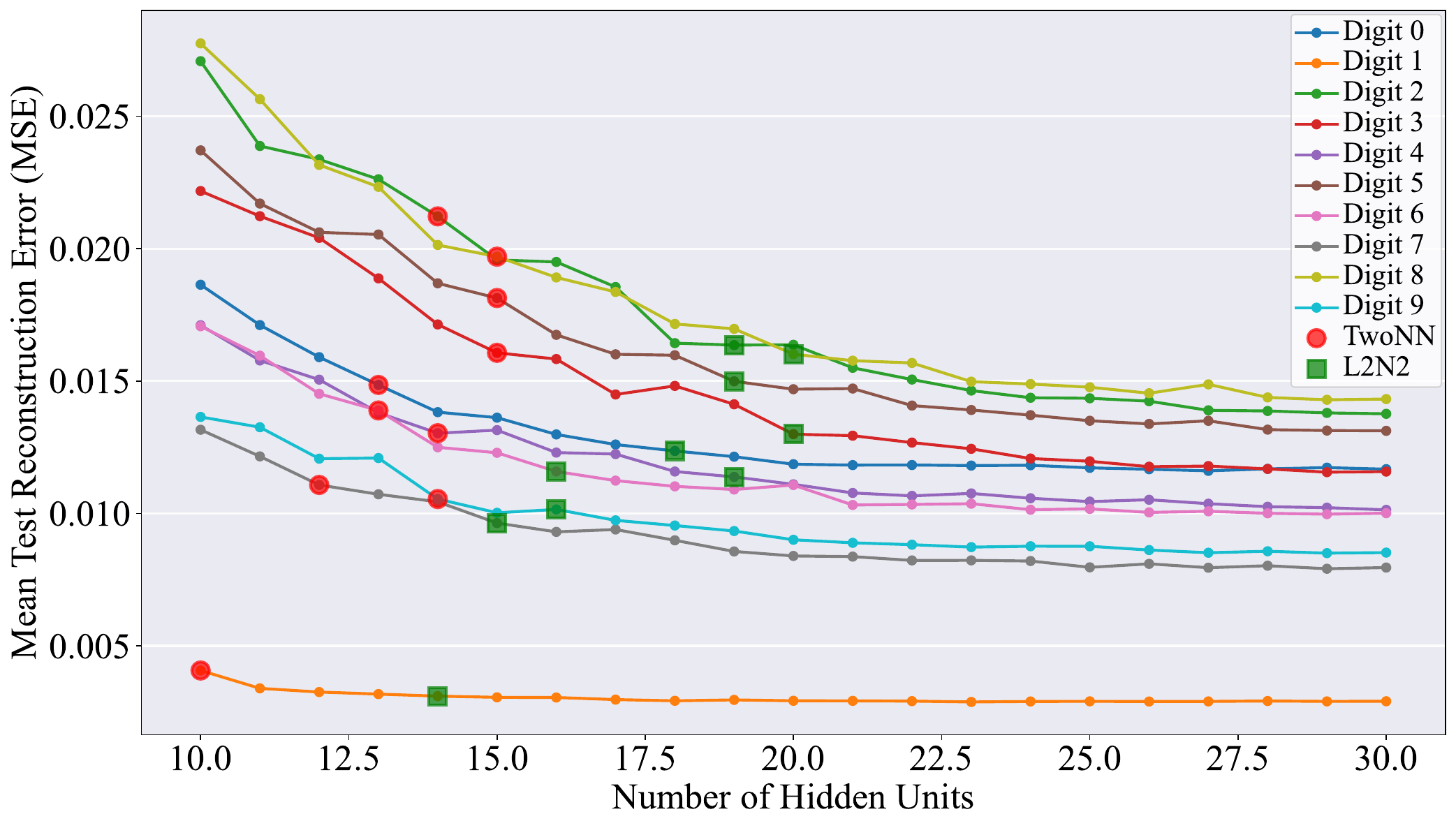}
    \caption{Downstream experiments results: Reconstruction errors (MSE) of autoencoders with bottleneck layers set to different values. }
    \label{fig:autoencoder}
\end{figure}
\vspace{5pt}

\section{Discussion}
\label{sec:exp_discussion}
In this section we provide a few further insights about {\name} and the results presented above.

\vspace{5pt}
\noindent {\bf Nearest-Neighbor Configurations:} An important observation from the benchmark results is that the simplest ratio configuration $(k,j) = (2,1)$ consistently gives the best results (here we presented $(4,2)$ and $(8,4)$, and see the supplementary materials for more examples). An explanation for this can be seen in Fig.\,\ref{fig:indiv_manifold_mpe}\,(a). We find that using ratios from more distant neighbors negatively impacts the MPE for low dimensional manifolds such as the one-dimensional helix. This is likely due to the low number of points available (2,500), causing the 8-th nearest neighbor to lie on another part of the helix, and thus skewing the estimated dimensionality. In future work, we plan to explore in what cases, if any, other choices for $(k,j)$ are preferable.   

\vspace{5pt}
\noindent{\bf Finte Sample Effects:}
Lemma \ref{lem:homogeneous} shows that the estimator is asymptotically unbiased at the log-scale.
Finite-sample deviations arise because the analysis relies on the  local homogeneity of the point process in the limit, which is rarely met in the finite setting. A theoretical analysis of the finite-sample bias remains as future work. Nevertheless, our experiments show that L2N2 exhibits much lower bias than competing methods (Fig. \ref{fig:sph_comparison}). 
More details and results of the effect of samples sizes on the coefficient estimates ($\alpha_{i,k}$ \& $\beta_{i,k}$) can be found in Appendix \ref{sec:finite_samp_analysis} .

\vspace{5pt}
\noindent{\bf Subsampling for Scaling:} For very large samples, we propose the following 
 efficient strategy. We replace \eqref{eqn:Lkj_avg} with the average value of $L_{k,j}$ taken over a small query subset. Note that the individual $L_{k,j}$ values are still computed with respect to the full dataset.
 The results when subsampling is used on the benchmark manifold experiments
 can be found in Table \ref{tab:subsamp_res}. It can be seen that subsampling in this manner increased the MPE by around 0.34\%-0.56\%, but with significant reduction in computational complexity.
 For comparison, when only 2,500 points are available in total, the MPE increases to 9.96\%. However, in this case, this computational optimization is unnecessary. 
 \begin{table}[]
    \centering
    \begin{tabular}{|c|c|c|c|c|}
    \hline
         No. Pts ($\times 10^3$) & $15$& $50$ & 100 & 500  \\
         \hline             
         \hline
         (SubSamp) MPE & 8.50 & 7.67 & 7.61 & 6.88 \\ 
         (Full Set) MPE & 8.15 & 7.33 & 7.05 & 6.41 \\
         \hline
         MPE Diff. & 0.45 & 0.34 & 0.56 & 0.47 \\
         \hline
    \end{tabular}
    \caption{Benchmark manifold MPE from the subsampling experiments (no rounding or optimization). (Full Set) indicates that all points are used. (SubSamp) indicates that a random subset of 2,500 points are used. }
    \label{tab:subsamp_res}
\end{table}

\vspace{5pt}
\noindent{\bf Comparison with the MLE:}
Earlier in the paper, we mentioned that the MLE (Levina-Bickel estimator) has similar properties to the {\name}, including that it universally converges to $d$ \cite{Penrose2011LimitTF}. However, 
the MLE performs significantly worse than  {\name} in practice.
The MLE results from \cite{binnie2025surveydimensionestimationmethods}
where nearest neighbor values of 
$k = 10, 20, 40$ and $80$ were tried. It was consistently found that $k=10$ yielded the best results for the MLE, which are the ones used 
in Table \ref{tab:mpe_benchmark}; larger $k$ did not yield better results.
Below we give a possible intuitive explanation for the gap in performance.

Consider the special case ($k=2$) of the MLE, given by
\[
\widehat d_{\mathrm{MLE}}(\cX_n) = \frac{1}{n}\sum_{i=1}^n \frac{1}{\log(R_2(X_i,\cX_n)/R_1(X_i,\cX_n))}.
\] 
Thus, we can write both estimators as
an average 
\[\frac1n\sum_{i=1}^n g(\log(R_2(X_i,\cX_n)/R_1(X_i,\cX_n)),\]
where for the MLE estimator we take $g(x) = 1/x$, while for the {\name} we take $g(x) = \log(1/x)$. 
We suggest that the inferior performance of the MLE  stems from the dynamic range of $1/x$,
which explodes for small $x$ and vanishes for large $x$. By contrast,
$\log(1/x)$ spreads mass more evenly across $\R$, suppressing both extremes.



\vspace{5pt}
\noindent{\bf Runtime Complexity:}
All the experiments in this paper were carried out on the CPU of the same Dell Desktop (Dell tower EBT2250, Intel Ultra 7, 16GB RAM). The main computational cost is the k-NN calculation (as in MLE, TwoNN, Gride). Computing the mean loglog-ratio scales linearly with the sample size, with minimal memory overhead. In practice, L2N2 is consistently faster due to its simple mean-based estimate. On benchmark manifolds, average runtime is 13/23/56 ms for 2.5k/5k/10k points, compared to TwoNN 76/303/1243 and MLE 23/53/139.

\vspace{5pt}
\noindent {\bf Future Work:} While {\name} provides generally accurate estimates,  the experiments do show significant effects from small sample sizes particularly in high intrinsic dimensions. In the future, we plan to to investigate more principled approaches to improve accuracy on small samples.
One such direction is to consider the entire distribution of $\bar L_{k,j}$ rather than its mean value. We believe it is possible to prove that the limiting distribution is universal, and to find exactly what this distribution is.
Another approach will be to devise a more systematic way to tune the parameters $\alpha_{k,j},\beta_{k,j}$, that will take the sample size and the different ID ranges into account.
Finally, we also plan to investigate the properties of the {\name} estimator for different values of $(k,j)$ in order to better understand in what cases can we benefit from using indexes other than $(2,1)$.

\appendices

\section{Evaluating the Integral in $C_{k,j}$} \label{app:eval}

In this section, we derive the last step for obtaining (3.2) by showing the following:

\begin{equation}
    \label{eqn:integral}
\int_0^\infty \log(s) e^{-(i+j)s}ds = \frac{-\gamma-\log(i+j)}{i+j}
\end{equation}
where $\gamma\approx 0.5772$ is the Euler–Mascheroni constant and $\log$ is the natural logarithm. This uses the Leibniz integral rule (also known as ``differentiating under the integral'').
Let $\phi=i+j$ and define
\begin{align*}
I(z,\phi) &= \int_0^\infty  s^{z-1} e^{-\phi s} ds \\
[u=\phi s \Rightarrow  ]&= \phi^{-z} \int_0^\infty  u^{z-1} e^{-u} du  = \phi^{-z}\; \Gamma(z),
\end{align*}
where $\Gamma$ is the Gamma function.
Taking the partial derivative with respect to $z$, on one side we obtain
\begin{align*}
\frac{\partial I(z,\phi)}{\partial z}  &= 
\frac{\partial }{\partial z}\left(\phi^{-z}\; \Gamma(z) \right) \\ 
&=\phi^{-z}\frac{\partial \Gamma(z)}{\partial z} - \phi^{-z}  \Gamma(z)\log(\phi). 
\end{align*}
On the other side, we obtain
\begin{align*}
\frac{\partial I(z,\phi)}{\partial z} &=\int_0^\infty \frac{\partial} {\partial z}\left(  s^{z-1} e^{-\phi s} \right)ds\\
&= \int_0^\infty s^{z-1}\log(s) e^{-\phi s} ds. 
\end{align*}
Setting $z=1$ and equating the two, we obtain 
\begin{align*}
 \int_0^\infty\log(s) e^{-\phi s} ds  &= \phi^{-1}\left(\left.\frac{\partial \Gamma(z)}{\partial z}\right|_{z=1} -  \Gamma(1)\log(\phi) \right)\\
 &=\frac{-\gamma - \log \phi}{\phi},
\end{align*}
as $\frac{\partial \Gamma(z)}{\partial z}$ evaluated at 1 is the negative of the Euler–Mascheroni constant, i.e. $-\gamma$ and $\Gamma(1)$ is 1. Substituting in $\phi=i+j$ gives the result as required.

Finally, we use \eqref{eqn:integral} in (3.2) to obtain
\[C_{k,j} = \binom{k-1}{j-1}\!\!\sum_{i=0}^{k-j-1}\!\!(-1)^i\binom{k\!-\!j\!-\!1}{i}\!\!\left(\frac{\gamma + \log (i+j)}{i+j}\right) . 
\]
We observe that for $k=2, j=1$, the sum simplifies and we obtain
$C_{2,1} = \gamma$. More generally, if $k = j+1$ then $C_{k,j} = \gamma + \log (j).$
A few other values of $C_{k,j}$ for different values of $k,j$ are given in the following table:
\begin{center}
\begin{tabular}{|c|c|c|}
\hline
$k$ & $j$ & $C_{k,j}$ \\
\hline\hline

2 & 1 & 0.57722 \\
3 & 1 & -0.05796 \\
5 & 1 & -0.14338 \\
5 & 2 & 0.03536 \\
6 & 3 & 0.14185 \\
8 & 4 & 0.10242 \\
8 & 7 & 2.52310 \\
10 & 8 & 0.85720 \\
\hline
\end{tabular}
\end{center}

\section{Additional MNIST Experimental Results}
The results for the digits ``2'', ``3'' and ''4'' of the MNIST dataset can be found in Fig.\,\ref{fig:mnist_234}. We find results that are consistent with Digit ``1'' in the main paper, where our predicted ID values are consistently higher than those of TwoNN and GriDE. We note that the gap between our predictions and TwoNN, GriDE is consistent with that shown in Fig.\,3 of the main paper, suggesting a possible underprediction of ID by TwoNN and GriDE.
\begin{figure*}
    \centering
    \begin{tabular}{ccc}
       \includegraphics[width=0.33\linewidth]{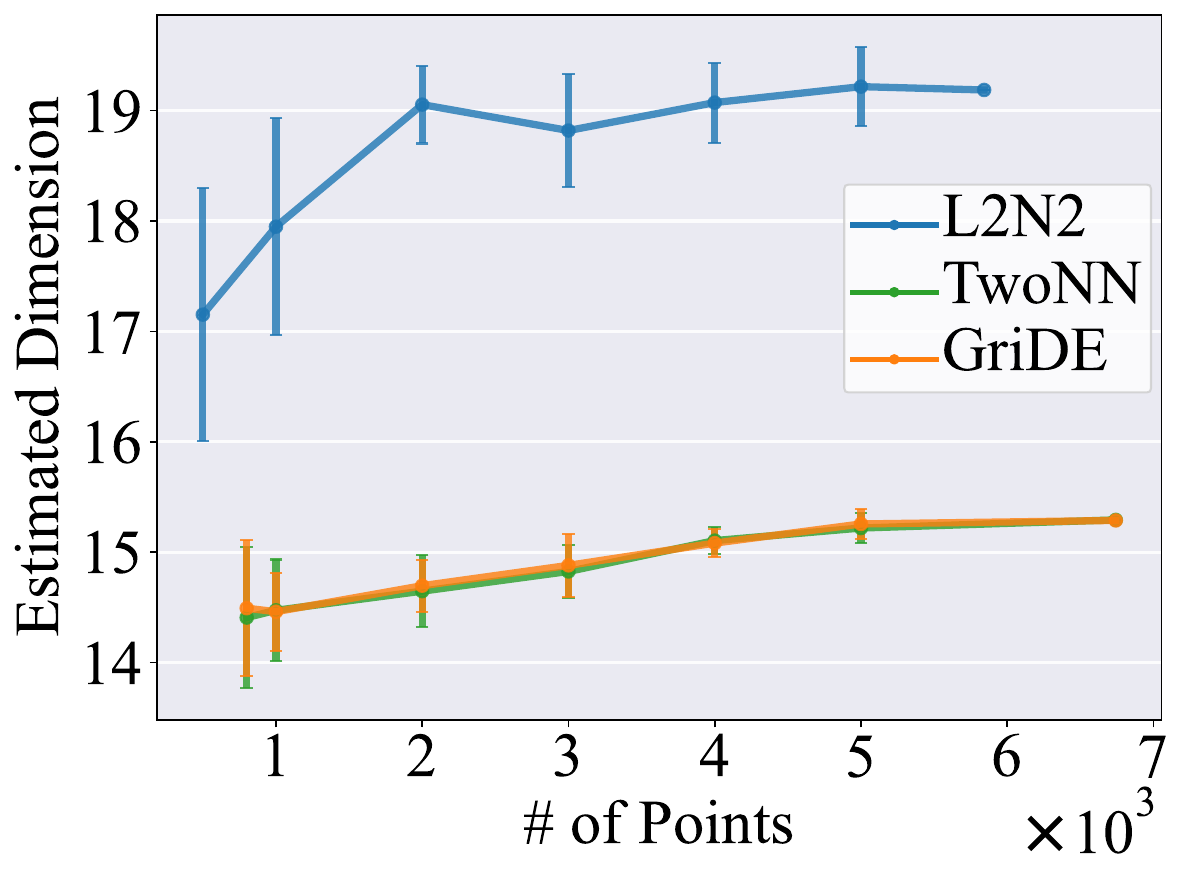}  &  \includegraphics[width=0.33\linewidth]{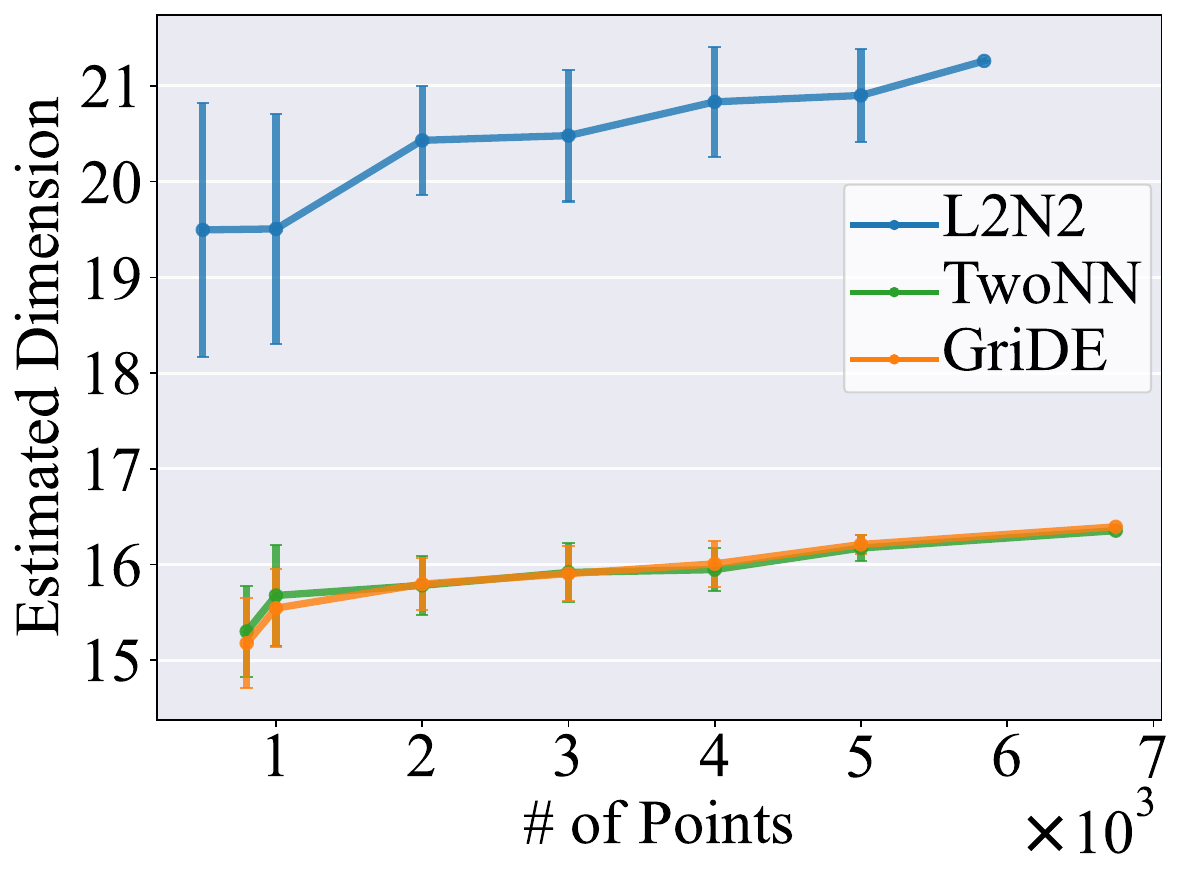} & \includegraphics[width=0.33\linewidth]{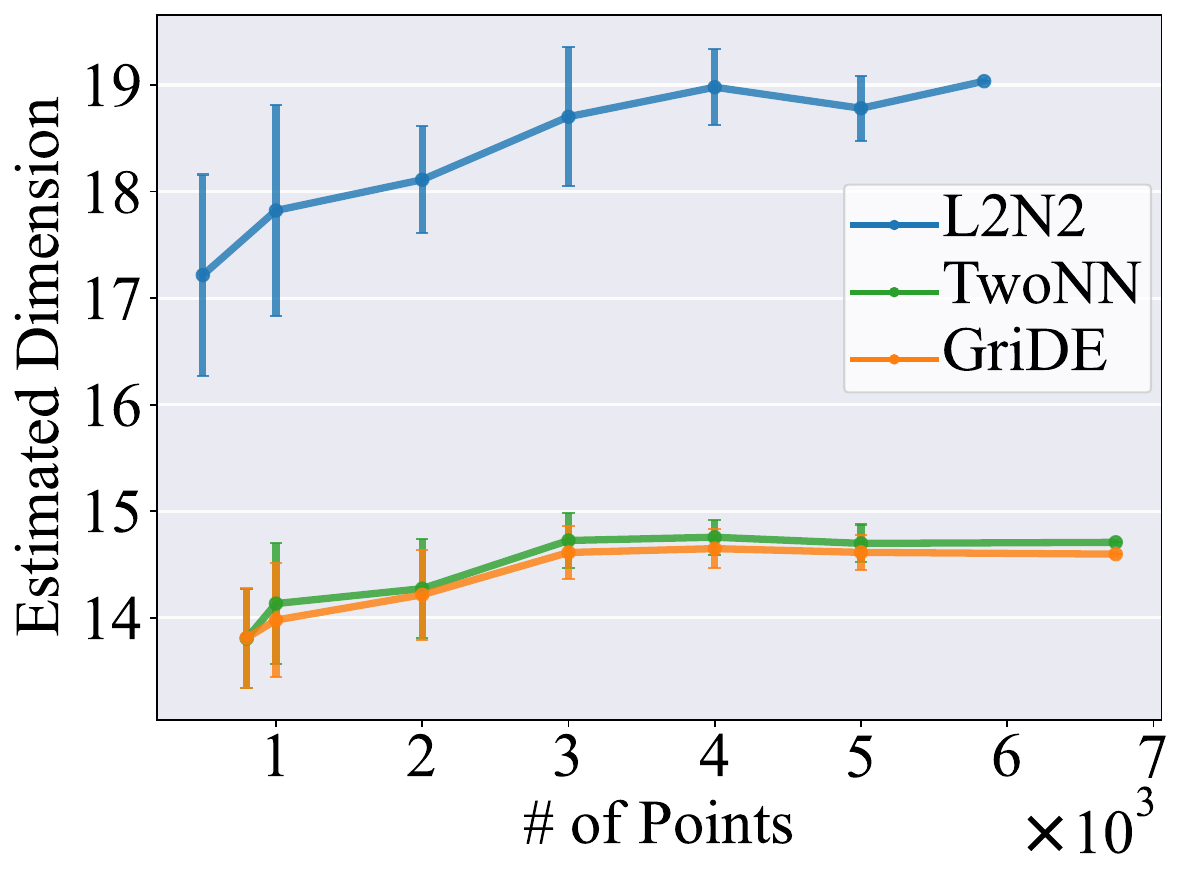}\\
        Digit ``2'' & Digit ``3'' & Digit ``4'' 
    \end{tabular}
    \caption{Results for digits ``2'',``3'', and ``4'' from the MNIST dataset. The error bars represent the standard deviation.}
    \label{fig:mnist_234}
\end{figure*}

\section{Empirical values for $\alpha^{(n)}_{2,1}$ and $\beta^{(n)}_{2,1}$}
\label{sec:finite_samp_analysis}
In this section, we provide more insight into the estimated values of $\alpha^{(n)}_{2,1}$ and $\beta^{(n)}_{2,1}$. We show how the number of points $n$, the choice of dimensions used in the estimation, and the models used affect the estimated values. For each value, we generated 50 point clouds for a range of dimensions and used estimated the parameters using linear regression. 

In Table~\ref{tab:gauss_1_20}, we show the estimated values for the $d$-dimensional Gaussian for $n$ ranging from 100 to 500,000. The values were estimated using dimensions ranging from 1 to 20 as in the benchmark experiments. For samples larger than 2,500 points, we estimate the mean by computing the average of the $\log \log$ of the nearest neighbor ratios of 2,500 randomly chosen points (with distances  computed with the all the points). 
For 500,000 points, the values of $\alpha^{(n)}_{2,1}$ and $\beta^{(n)}_{2,1}$ are close to the theoretical values of  1 and $\gamma=0.57721$,  respectively.
We also show the estimated parameters when the range of dimensions used was 10--40 in Table~\ref{tab:gauss_10_40}. 
In this case, the values are significantly different from the theoretical values but we do see $\alpha^{(n)}_{2,1}$ slowly increasing toward 1 and $\beta^{(n)}_{2,1}$ decreasing.  This indicates that the convergence rate is slower in higher dimensions. While this is not surprising, as the experiments show, one method of mitigating these effects is to estimate the values using an approximate range of dimensions. This can also be seen clearly in Fig. \ref{fig:gauss_diff_ranges}, where we plot the estimated values for using different dimension ranges for the $d$-dimensional Gaussian. While all the values are converging, it is significantly slower if we include higher dimensions. The error bars represent the standard deviation.

Finally, we illustrate the effect of different sampling models in Fig.~\ref{fig:comp_1_20} and~\ref{fig:comp_10_40}. In both cases the trend toward the theoretical values is clear, but in Fig.\,\ref{fig:comp_1_20} for 1000 points and up, the Gaussian seems to converge faster.

\begin{table}[h!]
\centering
\caption{\textbf{$d$-dimensional Gaussian}: Dimensions 1--20 used }
\label{tab:gauss_1_20}
\begin{tabular}{|c|c|c|}
\hline\rule[-1.2ex]{0pt}{4ex}
\# pts & $\alpha^{(n)}_{2,1}$ & $\beta^{(n)}_{2,1}$\\
\hline
\hline
100 & 0.843130 $\pm$ 0.030100 & 0.774468 $\pm$ 0.063900 \\
250 & 0.880452 $\pm$ 0.024624 & 0.730182 $\pm$ 0.052274 \\
500 & 0.904643 $\pm$ 0.028951 & 0.703754 $\pm$ 0.061459 \\
750 & 0.908416 $\pm$ 0.023954 & 0.702147 $\pm$ 0.050852 \\
1,000 & 0.911129 $\pm$ 0.025758 & 0.703216 $\pm$ 0.054681 \\
2,000 & 0.928568 $\pm$ 0.021527 & 0.675884 $\pm$ 0.045700 \\
2,500 & 0.929932 $\pm$ 0.022381 & 0.675130 $\pm$ 0.047513 \\
5,000 & 0.939984 $\pm$ 0.019841 & 0.666099 $\pm$ 0.042119 \\
10,000 & 0.947112 $\pm$ 0.019408 & 0.656437 $\pm$ 0.041201 \\
25,000 & 0.961028 $\pm$ 0.016408 & 0.633636 $\pm$ 0.034833 \\
50,000 & 0.965780 $\pm$ 0.015902 & 0.631085 $\pm$ 0.033758 \\
100,000 & 0.969699 $\pm$ 0.012250 & 0.621590 $\pm$ 0.026005 \\
250,000 & 0.977754 $\pm$ 0.011862 & 0.610773 $\pm$ 0.025181 \\
500,000 & 0.977408 $\pm$ 0.010183 & 0.613801 $\pm$ 0.021617 \\
\hline
\end{tabular}
\end{table}
\vspace{-0.2cm}
\begin{table}[h!]
\centering
\caption{\textbf{$d$-dimensional Gaussian}: Dimensions 10--40 used }
\vspace{-0.3cm}
\label{tab:gauss_10_40}
\begin{tabular}{|c|c|c|}
\hline\rule[-1.2ex]{0pt}{4ex}
\# pts & $\alpha^{(n)}_{2,1}$ & $\beta^{(n)}_{2,1}$\\
\hline
\hline
100 & 0.693905 $\pm$ 0.024272 & 1.150159 $\pm$ 0.076637 \\
250 & 0.708339 $\pm$ 0.011373 & 1.169491 $\pm$ 0.035910 \\
500 & 0.703530 $\pm$ 0.014324 & 1.221910 $\pm$ 0.045228 \\
750 & 0.727740 $\pm$ 0.016855 & 1.165622 $\pm$ 0.053221 \\
1,000 & 0.732002 $\pm$ 0.014138 & 1.165663 $\pm$ 0.044642 \\
2,000 & 0.752371 $\pm$ 0.016913 & 1.131156 $\pm$ 0.053402 \\
2,500 & 0.762282 $\pm$ 0.012706 & 1.106974 $\pm$ 0.040117 \\
5,000 & 0.768990 $\pm$ 0.015844 & 1.109482 $\pm$ 0.050026 \\
10,000 & 0.782166 $\pm$ 0.015283 & 1.087864 $\pm$ 0.048257 \\
25,000 & 0.804136 $\pm$ 0.015535 & 1.040560 $\pm$ 0.049052 \\
50,000 & 0.810475 $\pm$ 0.014195 & 1.037841 $\pm$ 0.044821 \\
100,000 & 0.826284 $\pm$ 0.014664 & 0.998123 $\pm$ 0.046302 \\
250,000 & 0.841989 $\pm$ 0.015627 & 0.963708 $\pm$ 0.049343 \\
500,000 & 0.851244 $\pm$ 0.015557 & 0.945327 $\pm$ 0.049121 \\
\hline
\end{tabular}
\end{table}
\begin{figure}[h!]
    \centering

    \begin{subfigure}{\linewidth}
        \centering
        \includegraphics[width=\linewidth]{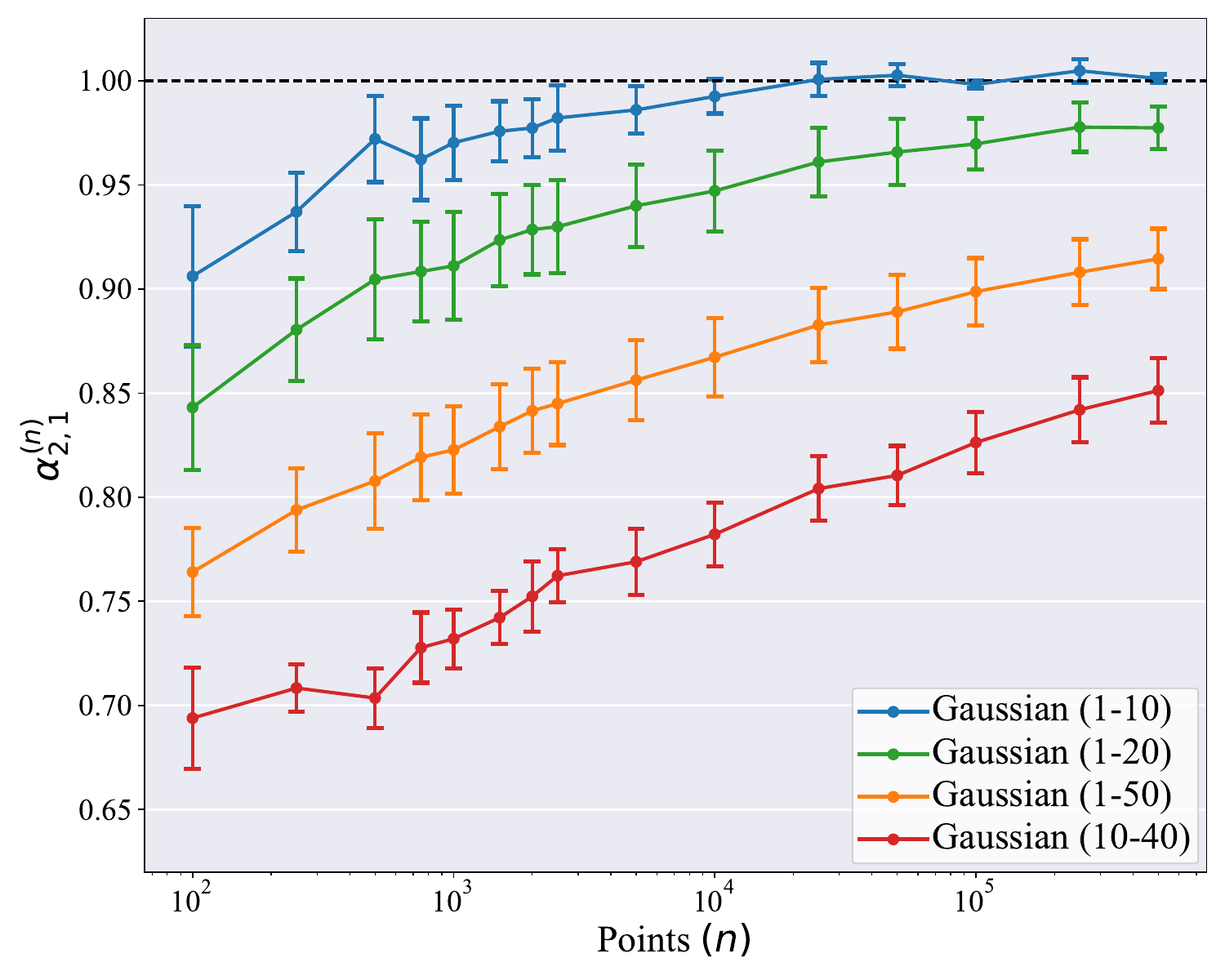}
        \caption{}
        \label{fig:gaussian_10_20_40_slopes}
    \end{subfigure}
    \vspace{-1em} 
    \begin{subfigure}{\linewidth}
        \centering
        \includegraphics[width=\linewidth]{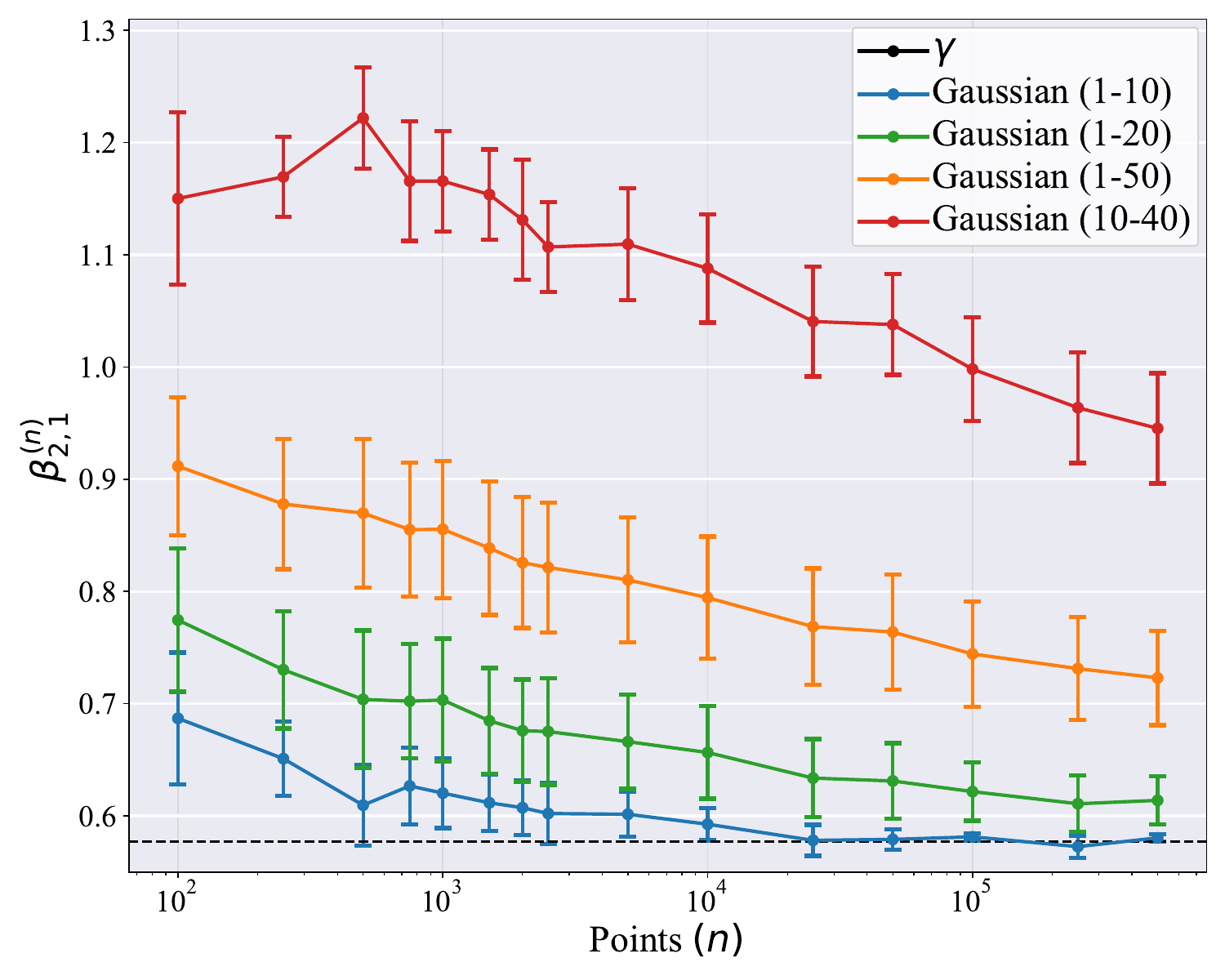}
        \caption{}
        \label{fig:gaussian_10_20_40_intercepts}
    \end{subfigure}
    \caption{A comparison of the estimated values using different dimensionality ranges for (a) $\alpha^{(n)}_{2,1}$ and (b) $\beta^{(n)}_{2,1}$  for different numbers of points $n$ sampling from a $d$-dimensional Gaussian distribution. The dotted black line indicates the theoretical limit value. The error bars represent the standard deviation.}
    \label{fig:gauss_diff_ranges}
\end{figure}

\begin{figure}[h!]
    \centering

    \begin{subfigure}{\linewidth}
        \centering
        \includegraphics[width=\linewidth]{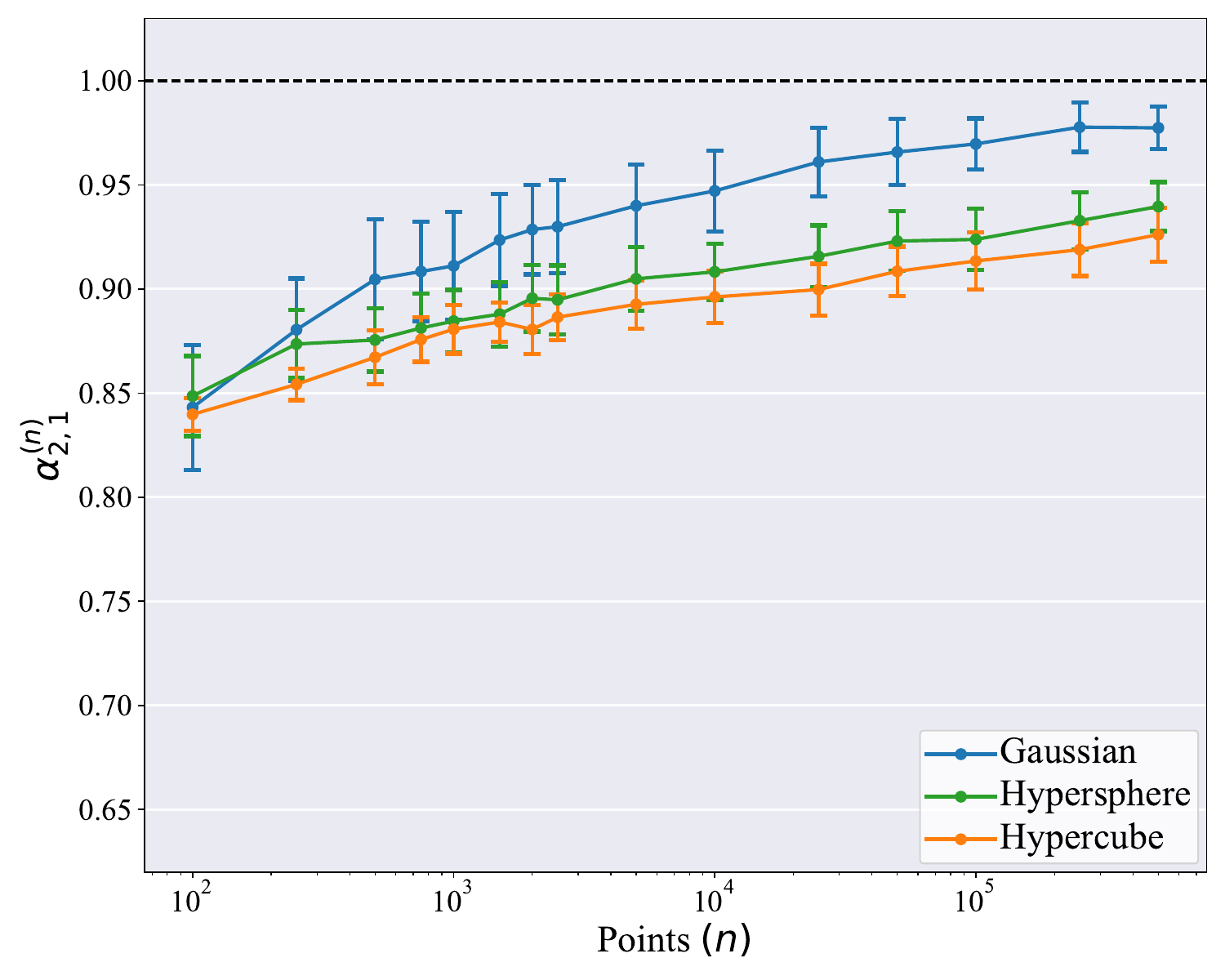}
        \caption{}
        \label{fig:comp_1_20_slope}
    \end{subfigure}
    \vspace{-1em} 
    \begin{subfigure}{\linewidth}
        \centering
        \includegraphics[width=\linewidth]{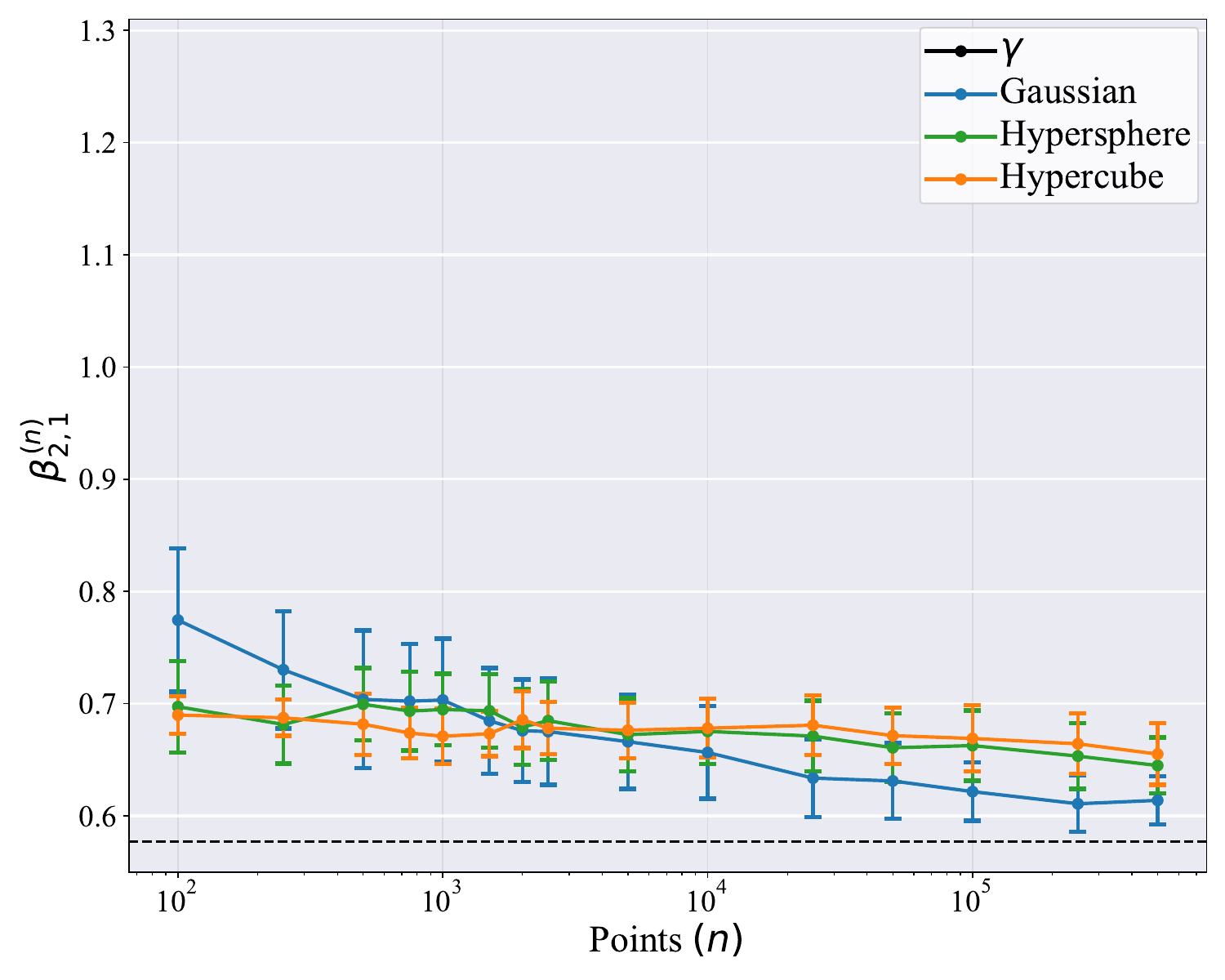}
        \caption{}
        \label{fig:comp_1_20_intercepts}
    \end{subfigure}
    \caption{A comparison of the estimated values using different sampling distributions (a) $\alpha^{(n)}_{2,1}$ and (b) $\beta^{(n)}_{2,1}$  for different numbers of points $n$ using dimensions up to $d=20$. The distributions shown are a $d$-dimensional Gaussian distribution, and uniform distributions on a $d$-dimensional hypercube and hypersphere. The dotted black line indicates the theoretical limit values. The error bars  represent  the standard deviation.}
    \label{fig:comp_1_20}
\end{figure}

\begin{figure}[h!]
    \centering

    \begin{subfigure}{\linewidth}
        \centering
        \includegraphics[width=\linewidth]{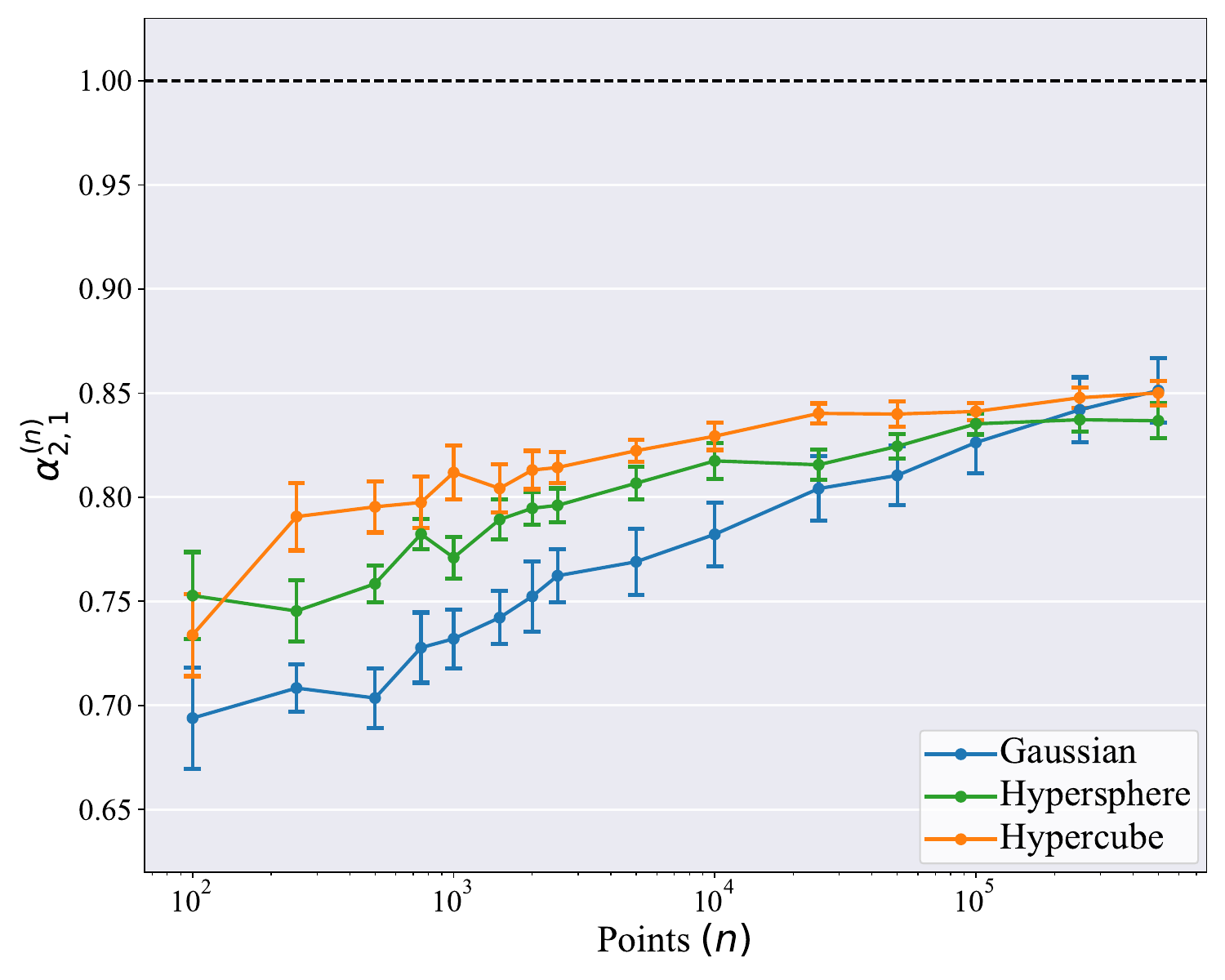}
        \caption{}
        \label{fig:comp_10_40_slope}
    \end{subfigure}
    \vspace{-1em} 
    \begin{subfigure}{\linewidth}
        \centering
        \includegraphics[width=\linewidth]{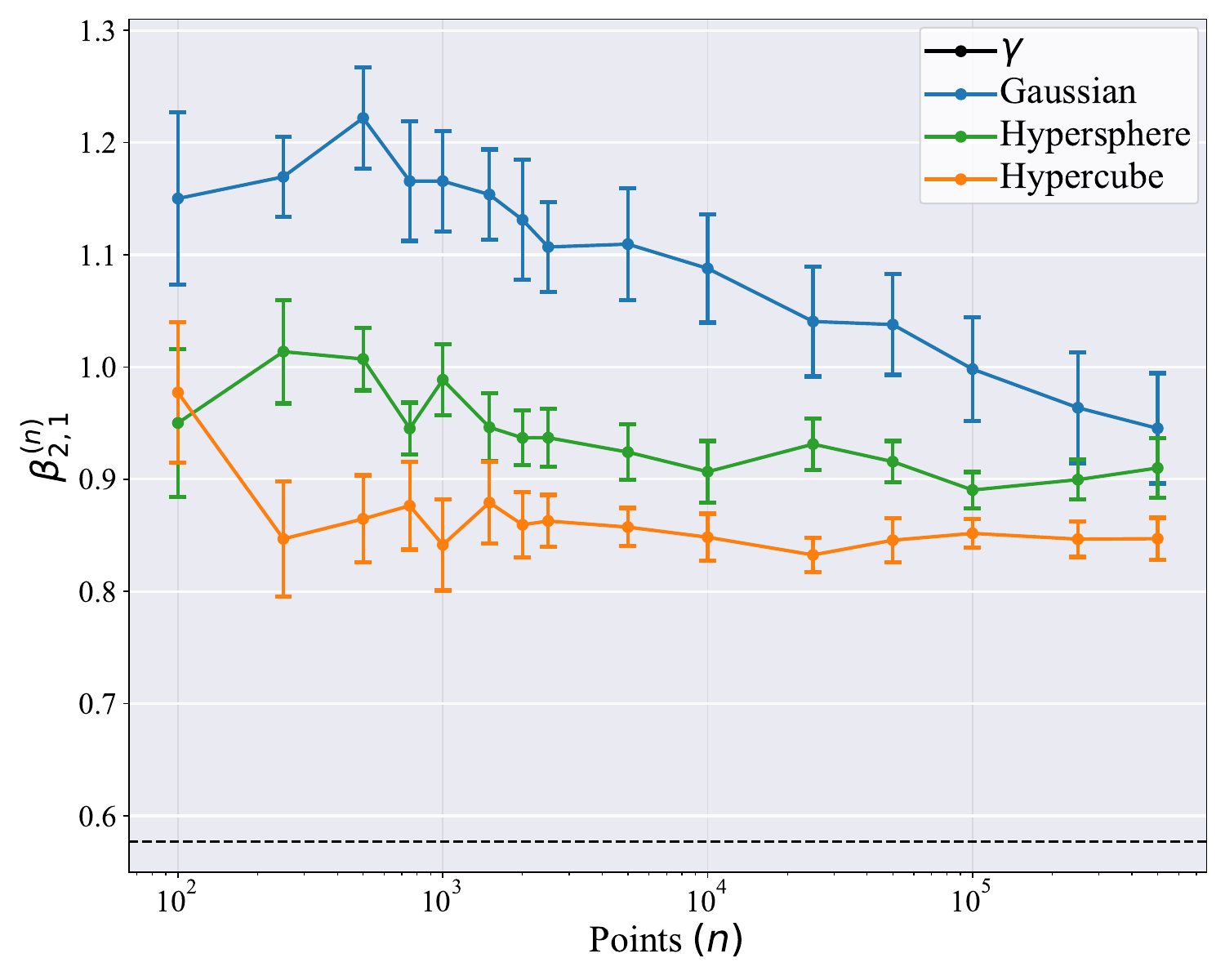}
        \caption{}
        \label{fig:comp_10_40_intercepts}
    \end{subfigure}
    \caption{A comparison of the estimated values using different sampling distributions (a) $\alpha^{(n)}_{2,1}$ and (b) $\beta^{(n)}_{2,1}$  for different numbers of points $n$ using samples from   dimensions 10--40. The distributions shown are a $d$-dimensional Gaussian distribution, and uniform distributions on a $d$-dimensional hypercube and hypersphere. The dotted black line indicates the theoretical limit values. The error bars represent the standard deviation.}
    \label{fig:comp_10_40}
\end{figure}
\bibliographystyle{plain}
    \bibliography{bibtex}

@book{penrose2003random,
  title={Random geometric graphs},
  author={Penrose, Mathew},
  volume={5},
  year={2003},
  publisher={OUP Oxford}
}

@article{bouligand1929notion,
  title={Sur la notion d’ordre de mesure d’un ensemble plan},
  author={Bouligand, G},
  journal={Bull. Sci. Math},
  volume={2},
  pages={185--192},
  year={1929}
}

@article{takens1981numerical,
  title={On the numerical determination of the dimension of an attractor},
  author={Takens, Floris},
  journal={Lect. Notes in Mathematics},
  volume={898},
  pages={366--381},
  year={1981},
  publisher={Springer-Verlag}
}

@article{grassberger1983characterization,
  title={Characterization of strange attractors},
  author={Grassberger, Peter and Procaccia, Itamar},
  journal={Physical review letters},
  volume={50},
  number={5},
  pages={346},
  year={1983},
  publisher={APS}
}

@article{carter2009local,
  title={On local intrinsic dimension estimation and its applications},
  author={Carter, Kevin M and Raich, Raviv and Hero III, Alfred O},
  journal={IEEE Transactions on Signal Processing},
  volume={58},
  number={2},
  pages={650--663},
  year={2009},
  publisher={IEEE}
}

@inproceedings{little2009estimation,
  title={Estimation of intrinsic dimensionality of samples from noisy low-dimensional manifolds in high dimensions with multiscale {SVD}},
  author={Little, Anna V and Lee, Jason and Jung, Yoon-Mo and Maggioni, Mauro},
  booktitle={2009 IEEE/SP 15th Workshop on Statistical Signal Processing},
  pages={85--88},
  year={2009},
  organization={IEEE}
}

@article{lung1988fractal,
  title={Fractal dimension measured with perimeter-area relation and toughness of materials},
  author={Lung, CW and Mu, ZQ},
  journal={Physical Review B},
  volume={38},
  number={16},
  pages={11781},
  year={1988},
  publisher={APS}
}

@article{zhou2021local,
  title={On local intrinsic dimensionality of deformation in complex materials},
  author={Zhou, Shuo and Tordesillas, Antoinette and Pouragha, Mehdi and Bailey, James and Bondell, Howard},
  journal={Scientific reports},
  volume={11},
  number={1},
  pages={10216},
  year={2021},
  publisher={Nature Publishing Group UK London}
}

@article{Facco,
author = {Facco, Elena and d’Errico, Maria and Rodriguez, Alex and Laio, Alessandro},
year = {2017},
month = {09},
pages = {},
title = {Estimating the intrinsic dimension of datasets by a minimal neighborhood information},
volume = {7},
journal = {Scientific Reports},
doi = {10.1038/s41598-017-11873-y}
}

@article{bobrowski2023universal,
  title={A universal null-distribution for topological data analysis},
  author={Bobrowski, Omer and Skraba, Primoz},
  journal={Scientific reports},
  volume={13},
  number={1},
  pages={12274},
  year={2023},
  publisher={Nature Publishing Group UK London}
}

@article{bobrowski2024universalityrandompersistenthomology,
  title={Universality in random persistent homology and scale-invariant functionals},
  author={Bobrowski, Omer and Skraba, Primoz},
  journal={arXiv preprint arXiv:2406.05553},
  year={2024}
}

@inproceedings{levina,
author = {Levina, Elizaveta and Bickel, Peter J.},
title = {Maximum Likelihood estimation of intrinsic dimension},
year = {2004},
booktitle = {Proceedings of the 18th International Conference on Neural Information Processing Systems},
pages = {777–784},
numpages = {8},
series = {NIPS'04}
}

@article{denti,
author = {Denti, Francesco and Doimo, Diego and Laio, Alessandro and Mira, Antonietta},
year = {2022},
month = {11},
pages = {20005},
title = {The generalized ratios intrinsic dimension estimator},
volume = {12},
journal = {Scientific Reports},
}

@article{Ceruti,
  title={{DANC}o: An intrinsic dimensionality estimator exploiting angle and norm concentration},
  author={Ceruti, Claudio and Bassis, Simone and Rozza, Alessandro and Lombardi, Gabriele and Casiraghi, Elena and Campadelli, Paola},
  journal={Pattern recognition},
  volume={47},
  number={8},
  pages={2569--2581},
  year={2014},
  publisher={Elsevier}
}

@article{Campadelli,
author = {Campadelli, P. and Casiraghi, E. and Ceruti, C. and Rozza, A.},
title = {Intrinsic Dimension Estimation: Relevant Techniques and a Benchmark Framework},
journal = {Mathematical Problems in Engineering},
volume = {2015},
number = {1},
pages = {759567},
year = {2015}
}

@ARTICLE{mnist,
  author={LeCun, Y. and Bottou, L. and Bengio, Y. and Haffner, P.},
  journal={Proceedings of the IEEE}, 
  title={Gradient-based learning applied to document recognition}, 
  year={1998},
  volume={86},
  number={11},
  pages={2278-2324},
}

@misc{isolet_54,
  author       = {Cole, Ron and Fanty, Mark},
  title        = {{ISOLET}},
  year         = {1991},
  howpublished = {UCI Machine Learning Repository},
  note         = {{DOI}: https://doi.org/10.24432/C51G69}
}

@article{costa2004geodesic,
  title={Geodesic entropic graphs for dimension and entropy estimation in manifold learning},
  author={Costa, Jose A and Hero, Alfred O},
  journal={IEEE Transactions on Signal Processing},
  volume={52},
  number={8},
  pages={2210--2221},
  year={2004},
  publisher={IEEE}
}

@INPROCEEDINGS{Costa,
  author={Costa, Jose A. and Hero, Alfred O.},
  booktitle={2004 12th European Signal Processing Conference}, 
  title={Learning intrinsic dimension and intrinsic entropy of high-dimensional datasets}, 
  year={2004},
  volume={},
  number={},
  pages={369-372},
  doi={}}

@article{ISOMAP,
author = {Joshua B. Tenenbaum  and Vin de Silva  and John C. Langford },
title = {A Global Geometric Framework for Nonlinear Dimensionality Reduction},
journal = {Science},
volume = {290},
number = {5500},
pages = {2319-2323},
year = {2000},
}

@article{Penrose2011LimitTF,
  title={Limit theory for point processes in manifolds},
  author={Mathew D. Penrose and Joseph E. Yukich},
  journal={Annals of Applied Probability},
  year={2013},
  volume={23},
  pages={2161-2211},
}

@article{whiteley2025statistical,
  title={Statistical Exploration of the Manifold Hypothesis},
  author={Whiteley, Nick and Gray, Annie and Rubin-Delanchy, Patrick},
  journal={Journal of the Royal Statistical Society: Series B},
  year={2025},
  publisher={Wiley-Blackwell}
}

@article{Fan2010IntrinsicDE,
  title={Intrinsic dimension estimation of data by principal component analysis},
  author={Mingyu Fan and Nannan Gu and Hong Qiao and Bo Zhang},
  journal={arXiv preprint arXiv:1002.2050},
  year={2010}
}

@article{binnie2025surveydimensionestimationmethods,
  title={A survey of dimension estimation methods},
  author={Binnie, James AD and D{\l}otko, Pawe{\l} and Harvey, John and Malinowski, Jakub and Yim, Ka Man},
  journal={arXiv preprint arXiv:2507.13887},
  year={2025}
}

@InProceedings{mindml,
author="Lombardi, Gabriele
and Rozza, Alessandro
and Ceruti, Claudio
and Casiraghi, Elena
and Campadelli, Paola",
editor="Gunopulos, Dimitrios
and Hofmann, Thomas
and Malerba, Donato
and Vazirgiannis, Michalis",
title="Minimum Neighbor Distance Estimators of Intrinsic Dimension",
booktitle="Machine Learning and Knowledge Discovery in Databases",
year="2011",
pages="374--389",
}

@article{coifman2006diffusion,
  title={Diffusion maps},
  author={Coifman, Ronald R and Lafon, St{\'e}phane},
  journal={Applied and computational harmonic analysis},
  volume={21},
  number={1},
  pages={5--30},
  year={2006},
  publisher={Elsevier}
}

@inbook{TLE,
author = {Laurent Amsaleg and Oussama Chelly and Michael E. Houle and Ken-ichi Kawarabayashi and Miloš Radovanović and Weeris Treeratanajaru},
title = {Intrinsic Dimensionality Estimation within Tight Localities},
booktitle = {Proceedings of the 2019 SIAM International Conference on Data Mining (SDM)},
pages = {181-189},
}

@inproceedings{PH,
author = {Birdal, Tolga and Lou, Aaron and Guibas, Leonidas and \c{S}im\c{s}ekli, Umut},
title = {Intrinsic dimension, persistent homology and generalization in neural networks},
year = {2021},
booktitle = {Proceedings of the 35th International Conference on Neural Information Processing Systems},
articleno = {519},
numpages = {14},
series = {NIPS '21}
}

@article{FisherS,
  title={Estimating the effective dimension of large biological datasets using Fisher separability analysis},
  author={Luca Albergante and Jonathan Bac and Andrei Yu. Zinovyev},
  journal={2019 International Joint Conference on Neural Networks (IJCNN)},
  year={2019},
  pages={1-8},
}

@ARTICLE{ESS,
  author={Johnsson, Kerstin and Soneson, Charlotte and Fontes, Magnus},
  journal={IEEE Transactions on Pattern Analysis and Machine Intelligence}, 
  title={Low Bias Local Intrinsic Dimension Estimation from Expected Simplex Skewness}, 
  year={2015},
  volume={37},
  number={1},
  pages={196-202},
}

@INPROCEEDINGS{knn,
  author={Costa, J.A. and Hero, A.O.},
  booktitle={The Thrity-Seventh Asilomar Conference on Signals, Systems and Computers, 2003}, 
  title={Entropic graphs for manifold learning}, 
  year={2003},
  volume={1},
  pages={316-320},
}

@InProceedings{WODCAP,
  title = 	 {{Dimensionality estimation without distances}},
  author = 	 {Kleindessner, Matthäus and Luxburg, Ulrike},
  booktitle = 	 {Proceedings of the Eighteenth International Conference on Artificial Intelligence and Statistics},
  pages = 	 {471--479},
  year = 	 {2015},
  volume = 	 {38},
  series = 	 {Proceedings of Machine Learning Research},
  publisher =    {PMLR},
}

@inproceedings{ConicalDim,
  title={Conical dimension as an intrisic dimension estimator and its applications},
  author={Yang, Xin and Michea, Sebastien and Zha, Hongyuan},
  booktitle={Proceedings of the 2007 SIAM International Conference on Data Mining},
  pages={169--179},
  year={2007},
  organization={SIAM}
}

@InProceedings{magnitude,
  title = 	 {Metric Space Magnitude and Generalisation in Neural Networks},
  author =       {Andreeva, Rayna and Limbeck, Katharina and Rieck, Bastian and Sarkar, Rik},
  booktitle = 	 {Proceedings of 2nd Annual Workshop on Topology, Algebra, and Geometry in Machine Learning (TAG-ML)},
  pages = 	 {242--253},
  year = 	 {2023},
  volume = 	 {221},
  series = 	 {Proceedings of Machine Learning Research},
  month = 	 {28 Jul},
  publisher =    {PMLR},
}

@ARTICLE{MDS,
  author={Sammon, J.W.},
  journal={IEEE Transactions on Computers}, 
  title={A Nonlinear Mapping for Data Structure Analysis}, 
  year={1969},
  volume={C-18},
  number={5},
  pages={401-409},
}

@Techreport{CIFAR100,
 author = {Krizhevsky, Alex and Hinton, Geoffrey},
 address = {Toronto, Ontario},
 institution = {University of Toronto},
 number = {0},
 publisher = {Technical report, University of Toronto},
 title = {Learning multiple layers of features from tiny images},
 year = {2009},
 title_with_no_special_chars = {Learning multiple layers of features from tiny images},
}

@inproceedings{Pope2021TheID,
title={The Intrinsic Dimension of Images and Its Impact on Learning},
author={Phil Pope and Chen Zhu and Ahmed Abdelkader and Micah Goldblum and Tom Goldstein},
booktitle={International Conference on Learning Representations},
year={2021},
}

@inproceedings{lorenz2023detecting,
  title={Detecting images generated by deep diffusion models using their local intrinsic dimensionality},
  author={Lorenz, Peter and Durall, Ricard L and Keuper, Janis},
  booktitle={Proceedings of the IEEE/CVF International Conference on Computer Vision},
  pages={448--459},
  year={2023}
}

@inproceedings{gamper2021multiple,
  title={Multiple instance captioning: Learning representations from histopathology textbooks and articles},
  author={Gamper, Jevgenij and Rajpoot, Nasir},
  booktitle={Proceedings of the IEEE/CVF conference on computer vision and pattern recognition},
  pages={16549--16559},
  year={2021}
}

@inproceedings{gong2019intrinsic,
  title={On the intrinsic dimensionality of image representations},
  author={Gong, Sixue and Boddeti, Vishnu Naresh and Jain, Anil K},
  booktitle={Proceedings of the IEEE/CVF Conference on Computer Vision and Pattern Recognition},
  pages={3987--3996},
  year={2019}
}

@article{Qiu_under,
title = {Underestimation modification for intrinsic dimension estimation},
journal = {Pattern Recognition},
volume = {140},
pages = {109580},
year = {2023},
author = {Haiquan Qiu and Youlong Yang and Hua Pan},
}

@article{Karkkainen,
title = {Additive autoencoder for dimension estimation},
journal = {Neurocomputing},
volume = {551},
pages = {126520},
year = {2023},
author = {Tommi Kärkkäinen and Jan Hänninen},
}

@inproceedings{Horvat,
 author = {Horvat, Christian and Pfister, Jean-Pascal},
 booktitle = {Advances in Neural Information Processing Systems},
 pages = {12225--12236},
 title = {Intrinsic dimensionality estimation using Normalizing Flows},
 volume = {35},
 year = {2022}
}

@inproceedings{Tempczyk,
  title={LIDL: Local Intrinsic Dimension Estimation Using Approximate Likelihood},
  author={Tempczyk, Piotr and Michaluk, Rafa{\l} and Garncarek, Lukasz and Spurek, Przemys{\l}aw and Tabor, Jacek and Golinski, Adam},
  booktitle={International Conference on Machine Learning},
  pages={21205--21231},
  year={2022},
  organization={PMLR}
}

\end{document}